\documentclass[letterpaper]{article} 
\usepackage{aaai2026}  
\usepackage{times}  
\usepackage{helvet}  
\usepackage{courier}  
\usepackage[hyphens]{url}  
\usepackage{graphicx} 
\usepackage{natbib}  
\usepackage{caption} 
\usepackage{algorithm}
\usepackage{algorithmic}
\usepackage{amsmath}
\usepackage{amssymb}
\usepackage{mathtools}
\usepackage{amsthm}
\usepackage{multirow}
\usepackage{enumitem}
\newtheorem{theorem}{Theorem}

\usepackage{booktabs}
\usepackage{newfloat}
\usepackage{listings}
\DeclareCaptionStyle{ruled}{labelfont=normalfont,labelsep=colon,strut=off} 
\floatstyle{ruled}
\newfloat{listing}{tb}{lst}{}
\floatname{listing}{Listing}
\title{Practical Global and Local Bounds in  Gaussian Process Regression via Chaining}
\author {
    Junyi Liu\textsuperscript{\rm 1},
    Stanley Kok\textsuperscript{\rm 1}
}
\affiliations {
    \textsuperscript{\rm 1}School of Computing, National University of Singapore, Singapore\\
    liujunyi@u.nus.edu, skok@comp.nus.edu.sg 
}
\begin{document}
\maketitle
\begin{abstract}
Gaussian process regression (GPR) is a popular nonparametric Bayesian method that provides predictive uncertainty estimates and is widely used in safety-critical applications. While prior research has introduced various uncertainty bounds, most existing approaches require access to specific input features, and rely on posterior mean and variance estimates or the tuning of hyperparameters. These limitations hinder robustness and fail to capture the model’s global behavior in expectation.
To address these limitations, we propose a chaining-based framework for estimating upper and lower bounds on the expected extreme values over unseen data, without requiring access to specific input features. We provide kernel-specific refinements for commonly used kernels such as RBF and Matérn, in which our bounds are tighter than generic constructions. We further improve numerical tightness by avoiding analytical relaxations. 
In addition to global estimation, we also develop a novel method for local uncertainty quantification at specified inputs. This approach leverages chaining geometry through partition diameters, adapting to local structures without relying on posterior variance scaling. 
Our experimental results validate the theoretical findings and demonstrate that our method outperforms existing approaches on both synthetic and real-world datasets. 
\end{abstract}
\begin{links}
    \link{Code}{https://github.com/Liu-Jun-Yi/Chaining-Bounds-in-GPR}
\end{links}

\section{Introduction}

Gaussian process regression (GPR), a flexible nonparametric Bayesian method, has emerged as a powerful tool for learning-based control, offering predictive uncertainty estimates through its posterior distribution. Classical works such as \citet{wu1993local} and \citet{schaback1999improved} provide deterministic error bounds in noiseless settings using reproducing kernel Hilbert space (RKHS) theory and fill-distance metrics. In contrast, the frequentist literature focuses on uncertainty bounds under noise, including time-uniform bounds based on information gain and RKHS norms for sequential decision tasks \citep{srinivas2010gaussian, srinivas2012information, chowdhury2017kernelized, whitehouse2023sublinear}.

In safety-critical domains like autonomous driving and robotics \citep{berkenkamp2019safe}, rigorous uncertainty quantification is essential. Although GPR is often used in these settings, existing methods offer limited theoretical guarantees. One approach combines posterior variance with robust control to construct uncertainty intervals. To improve robustness under model mismatch, recent works introduce probabilistic bounds via Lipschitz assumptions \citep{lederer2019uniform}, model-aware terms \citep{fiedler2021practical}, or hyperparameter variation \citep{capone2022gaussian}. Others calibrate posterior errors or use conformal prediction \citep{capone2023sharp, papadopoulos2024guaranteed}.

Despite this progress, most existing approaches aim to control errors or construct uncertainty intervals at specified input locations. These bounds typically require access to specific test input features, and rely on posterior mean and variance or intensive hyperparameter tuning, limiting their adaptability and often resulting in conservative or unstable intervals. 
Moreover, they do not directly address global behavior, such as the expected maximum bound, which is critical for assessing model reliability in safety-critical settings. For instance, autonomous systems must ensure trajectory deviations remain within safety thresholds over time, while disaster and financial risk models focus on expected peaks, such as the highest flood level or the worst market loss.

To address these limitations, we propose a novel framework that directly controls the global behavior of Gaussian processes without requiring access to specific input features. Specifically, we derive an expected upper bound on the global maximum of the process, which, by symmetry, also provides a corresponding lower bound. To our knowledge, this is the first application of chaining-based bounds in the context of GPR.

In practical applications, loose bounds can lead to inefficient decisions and resource waste, motivating the need for tighter estimates. We tighten the chaining bounds by exploiting properties of kernels, such as RBF and Matérn, yielding provably tighter and more practical bounds. Our implementation avoids analytical relaxations, such as loose constant factors, that are typically introduced for mathematical convenience but often result in overly conservative estimates.

While our method captures global behavior, we also propose a complementary method to quantify uncertainty at specific input features. 
Instead of relying on chaining accumulation, this method is inspired by chaining and leverages geometric partitioning to define locally adaptive reference sets. The uncertainty bounds are then computed by scaling the diameters of these partitions, aligning with the local geometry of the input space, and avoiding direct reliance on posterior mean and variance scaling.

\section{Background}
\subsection{Gaussian Process Regression}

Gaussian Process Regression (GPR) is a non-parametric Bayesian approach to regression~\citep{williams2006gaussian}. A Gaussian process (GP) is a collection of random variables \(\{f(x)\}_{x \in T}\), where any finite subset follows a multivariate normal distribution. We write \(f \sim \mathcal{GP}(m, K)\) to indicate that \(f\) is drawn from a GP with mean function \(m(x)\) and covariance function \(K(x, x')\). The covariance function defines dependencies between inputs and is typically chosen as the RBF or Matérn kernel. For simplicity, the mean function is often assumed zero.

\subsection{Chaining} \label{section22}

Chaining is a mathematical technique consisting of a succession of steps that provide successive approximations of an index space \( (T, d) \), where \( T \) is an index set or input set, and \(d\) is a metric on \(T\). Its fundamental idea is to group variables \(X_t\) (or, equivalently, their corresponding indices) that are nearly identical and approximate them at successive levels of granularity \citep{talagrand2014upper}. By doing this, we achieve tighter bounds, especially in cases where many variables are similar \citep{asadi2020chaining}. This approach mitigates the risk of large errors that can arise from such correlations.
%

To illustrate, consider a stochastic process \((X_t)_{t \in T}\), in which the difference between \(X_n\) and \(X_{0}\) is expressed as \(X_n - X_{0} = \sum_{t = 1}^{n} \left( X_t - X_{t-1} \right)\). 
When many variables \(X_t\) in \(T\) are nearly identical, strong correlations between them can obscure the true variation in the process. Grouping similar variables together helps reduce this redundancy by allowing us to approximate these highly correlated variables with a representative value, thereby simplifying the analysis and making the process easier to interpret and work with. A more detailed explanation can be found in Appendix~A.1.

For \(n \geq 0\), we select a subset \(T_n\), and for each \(t \in T\), we choose an approximation \(\pi_n(t)\) from \(T_n\). 
Using these \(\{\pi_n(t)\}\) points, we obtain the corresponding \(\{X_{\pi_n(t)}\}\) variables, which serve as successive approximations of \(\{X_t\}\). We start by assuming that \(T_0\) contains only one element \(t_0\), and thus \(\pi_0(t) = t_0\) for all \(t \in T\). The core relation is:
\( X_t - X_{t_0} = \sum_{n \ge 1} \left(X_{\pi_n(t)} - X_{\pi_{n-1}(t)}\right).\) 

This equality holds because, for sufficiently large \(n\), \(\pi_n(t)\) equals \(t\), meaning that beyond a certain point, the approximation stops, and the series becomes a finite sum. Specifically, as \(n\) increases, the sets \(T_n\) become progressively finer, eventually covering all points in \(T\). Once \(T_n\) contains \(t\), we have \(\pi_n(t) = t\), so no new information is added by further terms in the series. As a result, the infinite series truncates to a finite sum. This ensures convergence in practical settings where the process \(X_t\) is fully captured after a finite number of terms. The efficacy of this approach is rooted in the fact that for each approximation \(\pi_i(t)\), the variables \(X_t - X_{\pi_i(t)}\) are smaller than \(X_t - X_{t_0}\), making their supremum easier to handle. 
We present a simple example in Appendix~A.2.

This stepwise refinement converts the intractable global bound estimation into manageable local problems, simplifying the overall calculation, thus avoiding the complexity and error accumulation associated with global estimation.

\section{Related Work}
\label{related-work}

The concept of bounds in GPR originates from the Bayesian confidence intervals, which assume the target function is drawn from a Gaussian process prior
and are widely used in traditional GPR setups to reflect posterior uncertainty.

In contrast, the frequentist literature develops high-probability uncertainty bounds, which do not assume a prior but instead provide guarantees over repeated sampling. For example, \citet{srinivas2010gaussian}, \citet{srinivas2012information}, and \citet{chowdhury2017kernelized} derive time-uniform bounds on the prediction error by leveraging information gain and RKHS norms. These bounds are often used in Bayesian optimization and bandit settings. Similarly, \citet{fiedler2021practical} refine these bounds by introducing a model-misspecification-aware error term, still within a frequentist framework.
A different class of results involves deterministic error bounds from the field of scattered data approximation. Classical works such as \citet{wu1993local} and \citet{schaback1999improved}, as summarized in \citet{wendland2004scattered}, provide interpolation-based error bounds using fill distance and smoothness assumptions. These results are fundamentally non-probabilistic, rely on properties of the function space (e.g., RKHS regularity), and assume noiseless observations and the absence of distributional randomness.

Recent approaches adopt a Bayesian-style probabilistic setup. 
\citet{lederer2019uniform} introduce probabilistic bounds using Lipschitz continuity, which are derived on a finite grid and extended via covering arguments. Although their presentation is probabilistic, it deviates from the strict frequentist setting of time-uniform guarantees. Similarly, \citet{capone2022gaussian} design probabilistic error bounds based on a given hyperparameter range, aiming to mitigate the risk from inaccurate kernel choices.

Later, \citet{song2019distribution} propose distribution calibration by adjusting predictive distributions post-hoc using Gaussian processes and Beta calibration. \citet{capone2023sharp} introduce a posterior adjustment that sharpens GP intervals via variance calibration but lacks distribution-free coverage. More recently, \citet{papadopoulos2024guaranteed} develops conformal prediction to construct distribution-free intervals based on nonconformity scores.

These methods bound uncertainty pointwise but not global deviations. 
Our work advances GPR bounds in two directions. First, we implement Talagrand’s generic chaining bound to estimate the global expected supremum error, capturing expected global deviation rather than input-specific guarantees. While existing frequentist bounds upper-bound this supremum by maximizing over uncertainty intervals at given inputs, our method offers a theoretical alternative that targets the expected supremum. Second, we propose a new method for uncertainty quantification at specified inputs; Inspired by chaining geometry, it uses partitions to construct locally adaptive bounds around each input, reducing reliance on global posterior variance. 
Our methods jointly extend GPR bounds from local to global uncertainty control.

Chaining-based methods have recently gained attention in machine learning through Chaining Mutual Information (CMI)~\cite{asadi2018chaining}, which has been used to bound generalization error~\cite{clerico2022chained} and establish risk bounds for neural networks via hierarchical coverings~\cite{asadi2020chaining}. 
For GPR, we instead apply chaining directly to covariance functions, which naturally encode the process structure.

\section{Upper and Lower Bounds}

We now present our technical contributions. 
We avoid posterior variance scaling compared to existing methods and apply chaining directly to the distance metric induced by the kernel function. While GPR typically fits data and selects kernel hyperparameters, we use the kernel solely to define a distance metric, bypassing posterior inference. This eliminates dependence on posterior variance and calibration, yielding tighter, more reliable bounds.
We assume a bounded RKHS function and i.i.d. sub-Gaussian noise.

To formalize our analysis, we consider a Gaussian process \((X_t)_{t \in T}\), where each \(X_t\) follows a normal distribution with mean zero and variance \(\sigma^2\), and \(T\) is an index set (e.g., \(T \subseteq \mathbb{R}^n\)). 
For any two points \(s, t \in T\), the process satisfies \(\mathbb{E}[(X_s - X_t)^2] = d(s, t)^2\) and tail bound \(\mathbb{P}(|X_s - X_t| \geq u) \leq 2 \exp\left( -\frac{u^2}{2 d(s, t)^2} \right)\), where \(d(s, t)\) is the canonical pseudometric induced by the covariance function. Under this setting, we study two objectives: (i) estimating an expected upper bound
, and (ii) providing pointwise uncertainty bounds. 
We first introduce the expected upper bound, starting from Talagrand’s theorem~\citep{talagrand2014upper}.
\begin{theorem}\label{thm:tala} \citep{talagrand2014upper} (Eq 2.33)
Let \(T\) be an index set, \(t_0 \in T\) be an initial index,  \(T_n \subseteq T\) for \(n \ge 0\), and \(T_0 = \{t_0\} \). For each \(t \in T\), let  \(\pi_n(t) \in T_n\) for each \(n \geq 0\), where each \(\pi_n(t)\) represents a successive approximation of \(t\), and let \(\pi_n(t) = t \) for sufficiently large \(n\). 
Then
\begin{align*}
P\left( \sup_{t \in T} |X_t - X_{t_0}| > uS \right) \leq L \exp\left( -\frac{u^2}{2} \right), 
\end{align*}
where \(L\) is a universal constant (e.g., \( L = \exp(9/2) \) satisfies the condition), \(u \in \mathbb{R} \cup \{0\}\), \(d\!:\!T\!\times\!T\!\rightarrow\!\mathbb{R}\) is a distance metric on \(T\), and 
\begin{align*}
S := \sup_{t \in T} \sum_{n \geq 1} 2^{n/2} d(\pi_n(t), \pi_{n-1}(t)).
\end{align*}
\end{theorem}
Theorem~\ref{thm:tala} is purely theoretical, lacking practical implementation details. For instance, no method is provided for determining $S$,
\(\{T_n\}_{n \ge 0}\), \(\{\pi_n(t)\}_{n \ge 0}\), and \(t_0\). We address some of these deficiencies in subsequent sections. A generic bound that applies to all kernels is presented below.
\begin{theorem} \label{eqn:expsup}
(Talagrand's Generic Bound) Theorem~\ref{thm:tala}, combined with the formula 
\( \mathbb{E} \left[ Y \right] = \int_0^\infty P(Y \geq u) \, du , \)
which gives the expectation of a non-negative random variable, leads to the following upper bound:
\begin{align*}
\mathbb{E} \sup_{t \in T} X_t \leq  X_{t_0} + \mathbb{E}\left[ \sup_{t \in T} |X_t - X_{t_0}| \right] \\
\leq 
X_{t_0} + (1 + \sqrt{2})
L \sup_{t \in T} \sum_{n \geq 0} 2^{n/2}  d(t, T_n)
,
\end{align*}
where  $\mbox{$d(t, T_n)=\inf_{s \in T_n} \sqrt {K(t, t) + K(s, s) - 2K(t, s)}$}$
and \( t_0 \) is chosen such that \( X_{t_0} \) is close to zero due to the zero-mean property and the symmetry of the kernel.
\end{theorem}
\begin{proof}[Proof sketch.] 
We combine theorem~\ref{thm:tala} with the integral formula for non-negative expectations to get the upper bound \( \mathbb{E}[\sup_{t \in T} |X_t - X_{t_0}|] \leq LS \sqrt{\pi/2} \). The approximation \( d(t, \pi_n(t)) = \inf_{s \in T_n} d(t, s) \) and the triangle inequality control chaining increments.
Detailed proofs of Theorems ~\ref{thm:tala} and ~\ref{eqn:expsup} are in Appendices ~B.1 and ~B.2.
\end{proof} 
Unlike Talagrand’s purely theoretical formulation, we provide a concrete and fully implementable version of the chaining procedure, including explicit constructions, pseudocode, and runnable code. In addition, rather than relying on arbitrary constants sufficient for theoretical validity, we explicitly control quantities such as prefactors (e.g., $L$) to obtain tighter and more practically useful bounds, thereby enhancing both rigor and applicability.

 We further apply the general bounds to compute tighter bounds for specific kernels by deriving more precise estimates of \(\mathbb{E}\left[ \sup_{t \in T} |X_t - X_{t_0}| \right]\). The following subsections will first introduce the RBF and Matérn kernels, and then provide detailed proofs for their respective tighter bounds.

\subsection{Kernels}

In GPR, the distance between two input points is measured using a kernel, or covariance function, which quantifies their similarity in the feature space and defines the structure of the Gaussian process by controlling its smoothness and generalization ability. A widely used example is the radial basis function (RBF) kernel. It produces smooth, continuous estimates and is often combined with a constant kernel to model signal variance. It is defined as:
\[
K(s, t) = \sigma^2 \exp\left( -\frac{\|s - t\|^2}{2l^2} \right),
\]
where \( \|s - t\| \) is the Euclidean distance between the (multi-dimensional) input points \( s \) and \( t \), the \( \sigma^2 \) term represents the constant kernel, and \( l \) is the length-scale parameter.

The Matérn kernel is also a widely used covariance function that controls the smoothness of sampled functions through its parameter \( \nu > 0 \). It is defined as  
\(
K(s, t) = \frac{2^{1-\nu}}{\Gamma(\nu)} \left( \frac{\sqrt{2\nu} \|s - t\|}{l} \right)^{\nu} B_{\nu}\left( \frac{\sqrt{2\nu} \|s - t\|}{l} \right),
\)  
where \( l \) is the length scale, and \( B_{\nu} \) is the modified Bessel function of the second kind. Larger values of \( \nu \) imply smoother sample paths.  

When \( \nu = p + 1/2 \) with \( p \in \mathbb{N} \), the kernel simplifies to a product of an exponential and a polynomial of order \( p \)~\cite{seeger2004gaussian}. Commonly \( \nu = 3/2 \), giving  
\[
K(s, t) = \left(1 + \frac{\sqrt{3} \|s - t\|}{l} \right) \exp\left( -\frac{\sqrt{3} \|s - t\|}{l} \right).
\]
The distance between two points \( s \) and \( t \) in the context of GPs is \mbox{\( d(s, t) = \sqrt{\mathbb{E}[(X_s - X_t)^2]}\)},
where \( X_s \) and \( X_t \) are the values at points \( s \) and \( t \) respectively. This distance metric is derived from the covariance function \( K(s, t) \), which describes the covariance between the random variables \( X_s \) and \( X_t \). Specifically, it can be expanded as:
\begin{align*} \label{distance}
    d(s, t)^2 = \mathbb{E}[(X_s - X_t)^2] 
    =K(s, s) + K(t, t) - 2K(s, t).
\end{align*}
It is worth noting that \( s \) and \( t \) can each represent a vector describing a (multi-dimensional) input in a feature space, with \( X_s \) and \( X_t \) corresponding to the outputs evaluated at those input vectors. In this case, the covariance function \( K(s, t) \) reflects how similar the outputs are given their respective input vectors \( s \) and \( t \). 

\subsection{Tighter Bounds}
We will now show in Theorem~\ref{thm4}  how to refine the upper bound on \(\mathbb{E}\left[ \sup_{t \in T} |X_t - X_{t_0}| \right]\), yielding a tighter and more practical result for Gaussian processes with RBF kernels, compared to the bound in Theorem~\ref{eqn:expsup}. 

\begin{theorem}\label{thm2} (Tighter RBF Bound)
Consider a Gaussian process \((X_t)_{t \in T}\) with a radial basis function (RBF) kernel \(K(s, t) = \sigma^2 \exp\left( -\frac{\|s - t\|^2}{2l^2} \right)\), where \(T\) is an input/index set, \( \|s - t\| \) is the Euclidean distance between input points \( s \in T \) and \( t \in T \),  the term \( \sigma^2 \) represents the constant kernel, and \( l \) represents the length-scale parameter. 
Let \(t_0 \in T\) be an initial point, and \((T_n)_{n\ge 0}\) be a sequence such that \(T_n \subseteq T\). In addition, for each \(t \in T\), let \(\{\pi_n(t) \in T_n\}_{n \ge 0}\) represent a chain of successive approximations of \(t\) such that \( X_t - X_{t_0} = \sum_{n \ge 1} \left( X_{\pi_n(t)} - X_{\pi_{n-1}(t)} \right) \) with the condition that \(\pi_n(t) = t\) for sufficiently large \(n\) and \(\pi_0(t) = t_0\). Then
\begin{align*}
\mathbb{E} \sup_{t \in T} |X_t - X_{t_0}|  \leq  
LS
\leq
(1 + \sqrt{2}) L \sup_{t \in T} \sum_{n \geq 0} 2^{n/2} d'(t, T_n), 
\end{align*}
where 
\(d'(t, T_n) = \inf_{s \in T_n} \sqrt {K(t, t) + K(s, s) - 2 \sigma K^{\frac{1}{2}}(t, s)}\).
\end{theorem}
\begin{proof}
[Proof sketch.] We use the kernel-induced distance $d(s, t)^2 = 2\sigma^2 (1 - \exp(-\|s - t\|^2 / 2l^2))$, and apply the inequality $\|s - t\|^2 + \|t - u\|^2 \geq \|s - u\|^2 / 2$ with the convexity of the exponential to get $d(s, u)^2 \leq d'(s, t)^2 + d'(t, u)^2$, where $d'(s, t)^2 = 2\sigma^2 - 2\sigma K^{1/2}(s,t)$. We use the approximation and triangle inequality.
A detailed proof is given in Appendix~B.3.
\end{proof}
While the RBF kernel is widely used, other kernels, such as the Matérn kernel, are better suited for specific applications. In the following, we prove the upper and lower bounds for the Matérn kernel with \( \nu = 3/2 \). 
\begin{theorem}\label{thm3}(Tighter Matérn Bound) 
Consider a Gaussian process \((X_t)_{t \in T}\) with a  Matérn kernel 
$\mbox{\(K(s, t) = \left(1 + \frac{\sqrt{3} \|s - t\|}{l} \right) \exp \left(-\frac{\sqrt{3} \|s - t\|}{l} \right) \)}$,
where \(T\) is an input/index set, \( \|s - t\| \) is the Euclidean distance between input points \( s \in T \) and \( t \in T \), and  \(l\) is the length-scale parameter.  
Then
\begin{align*}
\mathbb{E} \sup_{t \in T} |X_t - X_{t_0}|   
\leq \\
(1 + \sqrt{2})  
L  \sup_{t \in T} \sum_{n \geq 0} 2^{\frac{n}{2}} [d''(t, T_n)+\sqrt{2}-2],
\end{align*}
where $\mbox{$d''(t, T_n))=\inf_{s \in T_n} \sqrt {K(t, t) + K(s, s) - 2K'(t, s)}$}$, and 
$\mbox{$K'(s,t) = \left(1 + \frac{\sqrt{3} \| s-t\|}{l} \right) \left[\exp \left(-\frac{\sqrt{3} \|s-t\|}{l} \right) - \frac{1}{2}\right]$}$. 
\end{theorem}
\begin{proof}[Proof sketch.]
By applying Chebyshev’s sum inequality to the sequences $(1 + x_i)$ and $\exp(-x_i)$ with $x_i = \frac{\sqrt{3} \| \cdot \|}{l}$, and leveraging the monotonicity of the function $f(x) = (1 + x)\exp(-x)$, we establish that the kernel satisfies a relaxed subadditivity condition
$
d(s,u)^2 \leq d'(s,t)^2 + d'(t,u)^2 - 2,
$
where $d'(s,t)^2 = 2 - 2K'(s,t)$.
We then follow the approximation and triangle inequality.
A detailed proof is provided in Appendix~B.4.
\end{proof}
By leveraging the kernel-dependent bounds in Theorems~\ref{thm2} and~\ref{thm3}, 
we provide both broadly applicable results and tighter, more practical bounds for Gaussian processes with commonly used kernels such as Matérn and RBF.

However, beyond kernel-specific flexibility, we also address a more fundamental limitation in classical chaining theory. Generic chaining is not optimized for sharp constants, resulting in loose conservative numerical bounds. In particular, classical results (e.g., the proof in Appendix~B.1) often include a large prefactor $L = \exp(9/2)$.

To obtain tighter estimates, we replace the constant $L$ by directly integrating the tail bound derived from chaining. Since the failure probability $p(u)$ is defined via a union bound over rare events and satisfies $p(u) \leq 1$, we truncate the integrand by $\min(p(u), 1)$, yielding a sharper estimate.
 \begin{theorem} \label{thm4} 
The expected supremum satisfies:
\begin{align*}
\mathbb{E} \sup_{t \in T} X_t \leq  X_{t_0} + \mathbb{E}\left[ \sup_{t \in T} |X_t - X_{t_0}| \right] \\
\leq X_{t_0} + S \int_0^\infty \min\left( \sum_{n \geq 1} 2^{2^{n+1}+1} \exp(-v^2 2^{n-1}), \, 1 \right) dv, 
\end{align*}
where 
\( S \leq (1 + \sqrt{2}) \sup_{t \in T} \sum_{n \geq 0} 2^{n/2} d'(t, s) \)
for the RBF kernel in Theorem~\ref{thm2}; 
\(S \leq (1 + \sqrt{2}) \sup_{t \in T} \sum_{n \geq 0} 2^{n/2} d''(t,s) \)
for the Matérn kernel in Theorem~\ref{thm3}.
\end{theorem}
\begin{proof}[Proof sketch.]
The Gaussian tail bound is applied to the increment \( X_{\pi_n(t)} - X_{\pi_{n-1}(t)} \), bounding its tail probability by \( \exp(-u^2 2^{n-1}) \), where the threshold is \( u\,2^{n/2} d(\pi_n(t), \pi_{n-1}(t)) \). A union bound \( p(u) \) is used for the failure probability of the event \( \Omega_u^c \).
 To avoid introducing large analytic prefactors $L$. We note that $p(u) \leq 1$ to get
$
\mathbb{E}[\sup_{t \in T} |X_t - X_{t_0}|] \leq \int_0^\infty \min(p(u/S), 1)\,du,
$, 
where $S$ is bounded via theorems~\ref{thm2} and~\ref{thm3}. A detailed proof is provided in Appendix~B.5.
\end{proof}

\subsection{Uncertainty Bounds}

In the application of chaining methods for uncertainty quantification in Gaussian process regression, relying solely on the expected supremum upper bound is often insufficient. It is also necessary to evaluate the probability of the supremum \emph{exceeding a given threshold}. To address this, we propose a method for deriving uncertainty bounds.

When incorporating test set features to predict pointwise bounds, the objective shifts from inferring the global supremum (and infimum) to estimating bounds for a specific test point, \(t'\). Symmetric intervals around a fixed reference (e.g., \(X_{t_0} \approx 0\)) may lack tightness, as the uncertainty should instead be symmetric around \(X_{t'}\). To address this, we focus on \(X_t - X_{\pi_n(t)}\), 
where \(X_{\pi_n(t)}\) is a reference point of \( X_t \) in \( T_n \),
yielding tighter and more precise bounds. Such bounds are useful when we wish to bound the 
uncertainty around the predicted value of an \textit{individual} test point. 
\begin{theorem} \label{thm5}
Let \(T\) be an index set, \(t_0 \in T\) be an initial index,  \(T_n \subseteq T\) for \(n \ge 0\), and \(T_0 = \{t_0\} \). For each \(t \in T\), let  \(\pi_n(t) \in T_n\) for each \(n \geq 0\), where each \(\pi_n(t)\) represents a successive approximation of \(t\), and let \(\pi_n(t) = t\) for sufficiently large \(n\). Then
\begin{align*}
P\left( |X_t - X_{\pi_n(t)}| < \sqrt{2 (\ln2 + \ln \frac{1}{\delta})}  d(t, T_n)\right) > 1 - \delta,
\end{align*}
where  $\mbox{$d(t, T_n)=\inf_{s \in T_n} \sqrt {K(t, t) + K(s, s) - 2K(t, s)}$}$.
\end{theorem}
\begin{proof}[Proof sketch.]
By introducing the confidence parameter 
\( \delta = 2 \exp(-u^2 2^{n-1}) \) and taking the natural logarithm of both sides, we have
$
\ln(\delta) = \ln(2) - u^2 2^{n-1}.
$
Rearranging for \( u^2 \), we get
$
u = \sqrt{\frac{\ln(2) + \ln(1/\delta)}{2^{n-1}}}.
$ 
and substituting this expression for \( u \) into the inequality for \( |X_t - X_{\pi_n(t)}| \) gives the result. A detailed proof is in Appendix B.6.
\end{proof}
In the chaining method, it is challenging to construct the sets $T_n$. \citet{talagrand2014upper} reformulates this as a geometric problem, replacing \(d(t, T_n)\) with partition diameters. This approach replaces the inherently combinatorial nature of supremum computations with geometric quantities that are easier to control and analyze. Partition diameters, as global quantities, reduce computational complexity to a manageable recursive process and provide better theoretical control.

To formalize this approach, an admissible sequence \((A_n)_{n \geq 0}\) is defined as an increasing sequence of partitions of \(T\), where \(A_0 = \{T\}\), and \(|(A_n)| \leq 2^{2^n}\) for \(n \geq 1\). Each \(A_n(t)\) denotes the unique set in \(A_n\) containing \(t \in T\).

From each partition \(A_n\), a representative point is selected from every set \(A \in A_n\) to construct the subset \(T_n \subseteq T\). By construction, for any \(t \in T\) and \(n \geq 0\), the distance between \(t\) and \(T_n\) satisfies
$
d(t, T_n) \leq \Delta(A_n(t)),
$
where \(\Delta(A_n(t))\) represents the diameter of \(A_n(t)\) under the metric \(d\). Thus Theorem~\ref{thm5} can be expressed as: 
\begin{align*}
P\left( |X_t - X_{\pi_n(t)}| < \sqrt{2 (\ln2 + \ln \frac{1}{\delta})} \Delta(A_n(t))\right) > 1 - \delta.
\end{align*}
This method retains the advantages of chaining without relying on posterior inference, providing tighter uncertainty pointwise bounds than global expected bounds.

\subsection{Algorithm of Our Chaining Method}
%
%
\begin{algorithm}[tb]
    \caption{Chaining Bounds Method}
    \label{alg:Chaining-Bounds-Method}
\textbf{Input}: Kernel function \( K(s, t) \) and dataset \(D\coloneq\{(t,X_t)\}\), where \(t\in \mathbb{R}^d\) is a $d$-dimensional input/index vector, and \(X_t\in\mathbb{R}\) is its associated output value.\\
\textbf{Output}:  $B$, a set containing the upper and lower bounds for each test example.

\begin{algorithmic}[1]
\STATE Split $D$ into a training set $D_{\text{train}}$ and a test set $D_{\text{test}}$.
\STATE Fit a Gaussian process using the kernel function \( K(s, t)\) to the training data $D_{\text{train}}$.
\STATE Initialize  \( T_0 \leftarrow  \{t_0\} \), \( T \leftarrow \{t: (t, \cdot) \in D_{\text{train}}\}\), and \(n_{\text{max}} \leftarrow \lfloor \log_2(\log_2(|T|)) \rfloor\). 
\FOR{\(n = 1\) to \(n_{\text{max}}\)}
    \STATE \( T_n \leftarrow T_{n-1} \)
    \WHILE{\(|T_n| < N_n = \min(2^{2^n}, |T|)\)}
        \STATE \(T_n \gets T_n \cup \{\arg\max_{t \in D_{\text{train}}} \min_{t^* \in T_n} d(t, t^*)\}\) 
    \ENDWHILE
    \STATE 
    \(A_{n, k} = \{t_i \in D_{\text{train}} \mid k =  \arg\min_{j} d(t_i, t_j), t_j \in T_n \}, \quad k = 1, \dots, N_n.\)
\ENDFOR

\STATE Initialize bounds \(B \gets \emptyset\).
\FOR{\(t \in D_{\text{test}}\)}
    \STATE Compute \(\mathbb{E} \sup_t |X_t - X_{t_0}|\) using Theorem~\ref{thm2} (RBF kernel)  or Theorem~\ref{thm3} (Matérn kernel).
    \STATE Find \(k^* = \arg\min_{k} \min_{t_i \in A_{n_{\text{max}}, k}} d(t, t_i)\).  
    \STATE \(\mu_{A_{n_{\text{max}}, k^*}} = \frac{1}{|A_{n_{\text{max}}, k^*}|} \sum_{t_i \in A_{L, k^*}} X_{t_i}\)
    \STATE \(\Delta(A_{n_{\text{max}}, k^*}) = \max_{t_i, t_j \in A_{n_{\text{max}}, k^*}} d(t_i, t_j)\)
    \STATE Compute uncertainty bound \(u(t) = \Delta(A_{n_{\text{max}}, k^*}) \cdot \sqrt{2 \cdot (\log(1/\delta) + \log(2))}\).
    \STATE \(B \leftarrow B \; \cup \; \{(\mu_{A_{n_{\text{max}}, k^*}} + u(t), \mu_{A_{n_{\text{max}}, k^*}} - u(t))\}\) for uncertainty bounds using Theorem~\ref{thm5}
    or  \\
    \(B \leftarrow B \; \cup \; \{( X_{t_0} + S \int_0^\infty \min\left( \sum_{n \geq 1} 2^{2^{n+1}+1} \exp(-v^2 2^{n-1}), \, 1 \right) dv , \;
    X_{t_0} - S \int_0^\infty \min\left( \sum_{n \geq 1} 2^{2^{n+1}+1} \exp(-v^2 2^{n-1}), \, 1 \right) dv.)\} \) for 
    unseen test points' bound using Theorem~\ref{thm4}.
\ENDFOR
\STATE \textbf{Return} \(B\).
    \end{algorithmic}
\end{algorithm}
In this work, we convert theoretical constructs into a practical chaining method for calculating the uncertainty bounds of GPR.
The full procedure is detailed in Algorithm~\ref{alg:Chaining-Bounds-Method}.

First, we preprocess the data by randomly dividing it into training and test sets. Then, we calculate the average of the output values (labels), and center the training set by subtracting the average from the output values of each example (now their mean is 0). Similarly, we subtract this training average value from the test set. Next, we fit a Gaussian process (GP) to the training data via maximum likelihood estimation to learn the parameters of the GP's kernel function and ensure that the kernel effectively models the data distribution.

Next, we construct the sequence of partitions \( A_n \) to progressively refine the training set. At each level \( n \), a representative set \( T_n \) is constructed by iteratively selecting points that maximize their minimum distance to the current representatives, ensuring that \( \sup_{t \in D_{\text{train}}} d(t, T_n) \) is minimized. Each training point is then assigned to the closest representative in \( T_n \), forming the partitions \( A_n = \{A_{n, k}\} \), where \( A_{n, k} \) contains all points nearest to the \( k \)-th representative.

We control the size of the set \(T_n\) by using the condition \( |T_n| \leq N_n \), where \(N_0 = 1\) and \(N_n = 2^{2^n}\) for \(n \geq 1\). This assumption leverages the approximation \( \sqrt{\log N_n} \approx 2^{n/2} \), which is a critical component in our analysis, and is related to the term \( \exp(-x^2) \), which governs the tails of a Gaussian distribution. Furthermore, the inequality \(N_n^2 \leq N_{n+1}\) demonstrates the effectiveness of this sequence in controlling the size of the sets \(T_n\) \cite{talagrand2014upper}.

Finally, we compute the distances between the test points and the representatives in \( T_n \) to determine their closest subsets in \( A_n \). For a test point \( t \), the mean of the training targets in its subset is used to compute a central prediction. The uncertainty bounds are derived by scaling the diameter of the subset with a factor proportional to \( \mathbb{E}[\sup_{t \in T} X_t]\) using Theorem~\ref{thm4} (for the upper and lower bounds over all unseen points) or \( \sqrt{2(\log(1/\delta) + \log(2))}  \Delta(A_n(t)) \) using Theorem~\ref{thm5} (for uncertainty quantification at specified inputs). Computational complexity and time costs are provided in Appendix~C.1.

\section{Experiment}
\label{sect:exp}

\subsection{Datasets}

The performance is evaluated on a synthetic dataset and five benchmark datasets that are widely used in prior studies:

\begin{itemize}[leftmargin=1em]
\setlength{\itemsep}{-3pt}  
    \item 
    \textbf{Synthetic Data} consists of 50 random functions sampled from a RKHS over \( D = [-1, 1] \) and evaluated at 1000 points. Each function combines kernels centered at random points, with 50 noisy samples drawn (Gaussian noise, standard deviation 0.5).
    \item 
    \textbf{Boston Housing}~\citep{cournapeau2007scikit} contains 506 samples, each with 13 features (e.g., crime rates and pollution) predicting the median house price.
    \item 
    \textbf{Sarcos}~\citep{schaal2009sl} consists of 44,484 training and 4,449 test samples, with each sample containing 21 input features from a seven-degree-of-freedom robotic arm. The task is to predict torque at seven joints.
    \item 
    \textbf{USGS Earthquake}~\citep{usgs_earthquake} contains thousands of observations on earthquake occurrences, detailing the time, location, and magnitude of significant seismic events recorded by the U.S. Geological Survey.
    \item 
    \textbf{Loa\_CO2}~\citep{mauna_loa_co2} contains CO2 concentration measurements from the Mauna Loa Observatory in Hawaii. 
    Its inputs include date and CO2 concentration.
    \item 
    \textbf{Auto-mpg}~\citep{auto_mpg} is a dataset focused on fuel consumption measured in miles per gallon (MPG). The original dataset consists of 398 observations, with 6 missing values discarded. It includes 7 input features such as engine capacity and number of cylinders.  
\end{itemize}

\subsection{Evaluation Metrics}

The performance of our proposed approach is evaluated using standard metrics for prediction intervals, as described by \citet{khosravi2010lower}.
\begin{itemize}
   \item 
   \textbf{Prediction Interval Coverage Probability (PICP)}. 
   This metric evaluates the percentage of test observations that lie within the bounds of the prediction intervals (PIs) at a given confidence level \((1 - \alpha) \%\)
   . It is calculated as 
        \underline{$ \text{PICP} = \frac{1}{n} \sum_{i=1}^{n} c_i$},
    where \( c_i = 1 \) if the output at point \( i \) lies within the bounds \( [L(X_i), U(X_i)] \), and \( c_i = 0 \) otherwise. Here, \( L(X_i) \) and \( U(X_i) \) denote the lower and upper bounds of the \(i^{\text{th}}\) PI. 
    \item 
    \textbf{Normalized Mean Prediction Interval Width (NMPIW)}. PIs that are too wide provide little useful information, so the NMPIW metric quantifies the width of the PIs as
      \underline{$    \text{NMPIW} = \frac{\frac{1}{n} \sum_{i=1}^{n} (U(X_i) -L(X_i))}{R}$},
    where \( R \) is the range of the target variable. 
    NMPIW expresses the average PI width as a percentage of the target range.
    \item 
    \textbf{Coverage Width-Based Criterion (CWC)}. This is the \textit{primary} evaluation metric because it balances the conflicting goals of achieving narrow PIs (low NMPIW) and high coverage (high PICP). (Note that a good PICP score can be trivially achieved at the expense of NMPIW (by using overly wide PIs) and vice versa (by using overly narrow PIs). Hence, either PICP or NMPIW alone is insufficient to completely reflect the goodness of bounds.) CWC is defined as
        \underline{$ \text{CWC} =  \text{NMPIW} \left( 1 + \gamma ( \text{PICP}) e^{-\eta (PICP - \mu)} \right)$},
  where \( \gamma \) is a hyperparameter and \( \eta = 50 \) to penalize narrow intervals; and \( \mu \) represents the nominal confidence level (\( \mu=1\) for extremal bounds). When \( \text{PICP} \geq \mu \), \( \gamma = 0 \); otherwise, \( \gamma = 1 \).
\end{itemize}

\subsection{Baseline Settings}

We compare our chaining method to the following three state-of-the-art baselines, previously introduced in the related-work section: (i) 
\textbf{Lederer19}~\citep{lederer2019uniform}, which introduces probabilistic Lipschitz constants to reduce the reliance on prior knowledge, estimates errors on a finite grid, and extends them to the input space;  (ii) \textbf{Fiedler21}~\citep{fiedler2021practical}, which modifies its objective bound function by introducing an error term based on the work of \citep{chowdhury2017kernelized};  $\mbox{(iii) \textbf{Capone22}}$~\citep{capone2022gaussian}, which tackles hyperparameter misspecification by proposing a method to calculate error bounds across a given range of hyperparameters; and (iv)~\textbf{Bayesian CI}, which uses the standard GP posterior mean and standard deviation to form Bayesian credible intervals $\mu(x) \pm z \cdot \sigma(x)$.

The baselines are evaluated on two tasks: (1) pointwise uncertainty estimation using per-point confidence intervals, and (2) estimation of upper and lower bounds for the expected maximum and minimum values on unseen data.
 We set $\delta$ to the confidence level in task~(1), and to $0.0001$ in task~(2) to approximate the ideal case $\delta \to 0$ required for uniform bounds. 
Bayesian CI is only evaluated on task~(1), as it requires a fixed test point \( x \) to provide \( P(f(x) \in [\mu(x) \pm z \cdot \sigma(x)]) \ge 1 - \delta \), while task~(2) requires a uniform guarantee for \( \forall x \in \mathcal{X} \), which the others explicitly provide.

For Lederer19, we use the default implementation with 10000 grid points over a fixed domain and resolution $\tau = 10^{-8}$. For Fiedler21, we compare the default RKHS norm $B = 2$ with a data-driven estimate $y^\top(K+\lambda I)^{-1}y$, using the better one, and fix the noise level to $0.0001$. For Capone22, we use the provided hyperparameter grids when available, or adopt ranges from similar datasets otherwise. All methods use the RBF kernel as in their original versions.

\subsection{Results}
We perform experiments on the two tasks defined above. Table~\ref{tab:confidence_comparison} compares the performance of our method with baseline methods in terms of CWC values at the 99\%, 95\%, and 90\% confidence levels for conventional prediction on each test point. 
Our methods consistently achieve the lowest CWC values, demonstrating superior coverage while providing compact intervals. Complete results, including PICP, NMPIW, and CWC values, are provided in Appendix~C.2.
For the second experiment, 
Table~\ref{tab:performance_comparison_sup} presents the results at the expected upper and lower bounds. Complete results, including PICP, NMPIW, and CWC values, are provided in Appendix~C.3. 
From Table~\ref{tab:performance_comparison_sup}, it can be observed that our method produces the tightest bounds and has the best (lowest) CWC values in 5 of the 6 datasets, and is a close second on the last one.
\begin{table}[h!]
\centering
\setlength{\tabcolsep}{2.5pt}
\begin{tabular}{l|ccc|ccc}
\toprule
\textbf{Method} 
& \multicolumn{3}{c|}{\textbf{Synthetic}} 
& \multicolumn{3}{c}{\textbf{Boston}} \\
& 99\% & 95\% & 90\% & 99\% & 95\% & 90\% \\
\midrule
RBF (Ours)      & 1.80 & 1.50 & 1.35 & 1.01 & 0.84 & 0.76 \\
Matérn (Ours)   & \textbf{1.80} & 1.50 & \textbf{1.35} & \textbf{0.76} & \textbf{0.64} & \textbf{0.57} \\
Capone22        & 5.30 & 7.02 & 2.63 & 1.30 & 1.28 & 1.18 \\
Fiedler21       & 3.95 & \textbf{1.49} & 1.49 & 3.46 & 3.46 & 3.46 \\
Lederer19       & 3.63 & 3.56 & 3.53 & 1.47 & 1.35 & 1.41 \\
Bayesian CI     & 37.19 & 69.02 & 113.08 & 5.95 & 4.94 & 2.66 \\
\midrule
\textbf{Method} 
& \multicolumn{3}{c|}{\textbf{Sarcos}} 
& \multicolumn{3}{c}{\textbf{USGS\_EQ}} \\
& 99\% & 95\% & 90\% & 99\% & 95\% & 90\% \\
\midrule
RBF (Ours)      & 0.48 & 0.40 & 0.36 & \textbf{2.69} & 2.25 & 2.03 \\
Matérn (Ours)   & \textbf{0.40} & \textbf{0.33} & \textbf{0.30} & 2.76 & \textbf{2.25} & \textbf{2.03} \\
Capone22        & 0.53 & 0.51 & 0.50 & 8.41 & 8.41 & 8.41 \\
Fiedler21       & 1.42 & 1.42 & 1.42 & 3.26 & 3.26 & 3.26 \\
Lederer19       & 0.56 & 0.50 & 0.49 & 4.23 & 4.13 & 4.07 \\
Bayesian CI     & 2.03 & 3.44 & 3.22 & 6.13 & 3.00 & 2.82 \\
\midrule
\textbf{Method} 
& \multicolumn{3}{c|}{\textbf{Loa\_CO2}} 
& \multicolumn{3}{c}{\textbf{Auto-mpg}} \\
& 99\% & 95\% & 90\% & 99\% & 95\% & 90\% \\
\midrule
RBF (Ours)      & 0.81 & 0.68 & 0.61 & 1.09 & 0.91 & 0.82 \\
Matérn (Ours)   & \textbf{0.24} & \textbf{0.20} & \textbf{0.18} & \textbf{0.84} & \textbf{0.70} & \textbf{0.63} \\
Capone22        & 1761.85 & 239.41 & 20.68 & 6.85 & 3.23 & 1.95 \\
Fiedler21       & 3.71 & 3.71 & 3.71 & 1.39 & 1.39 & 1.39 \\
Lederer19       & 0.55 & 0.53 & 0.52 & 50.03 & 48.42 & 47.70 \\
Bayesian CI     & 7.31 & 3.30 & 1.54 & 3.51 & 2.34 & 0.93 \\
\bottomrule
\end{tabular}
\caption{Comparison of CWC at 99\%, 95\%, and 90\% confidence levels across six datasets.}
\label{tab:confidence_comparison}
\end{table}
\begin{table}[h]
\centering
\begin{tabular}{l|c|c|c}
\toprule
        & \textbf{Synthetic} & \textbf{Boston} & \textbf{Sarcos} \\
\midrule
RBF(Ours)       & 1.68 & 1.75 & 1.03 \\
Matérn(Ours)    & \textbf{1.67} & \textbf{1.64} & \textbf{0.78} \\
Capone22        & 4.89 & 1.77 & 1.18 \\
Fiedler21       & 2.20 & 5.04 & 2.31 \\
Lederer19       & 2.02 & 1.78 & 1.34 \\
\midrule
        & \textbf{USGS EQ} & \textbf{Loa\_CO2} & \textbf{Auto-mpg} \\
\midrule
RBF(Ours)       & 2.59 & \textbf{1.70} & 3.06 \\
Matérn(Ours)    & \textbf{2.56} & 2.08 & 3.24 \\
Capone22        & 4.62 & 2.07 & \textbf{2.81} \\
Fiedler21       & 2.57 & 16.01 & 7.16 \\
Lederer19       & 3.07 & 1.73 & 57.22 \\
\bottomrule
\end{tabular}
\caption{Comparison of CWC across synthetic and real-world datasets of the expected upper and lower bounds over unseen points only using the training set.}
\label{tab:performance_comparison_sup}
\end{table}
\begin{figure}[h]
\begin{center}
\fbox{%
    \begin{minipage}[b]{0.15\textwidth}
        \includegraphics[width=\linewidth]{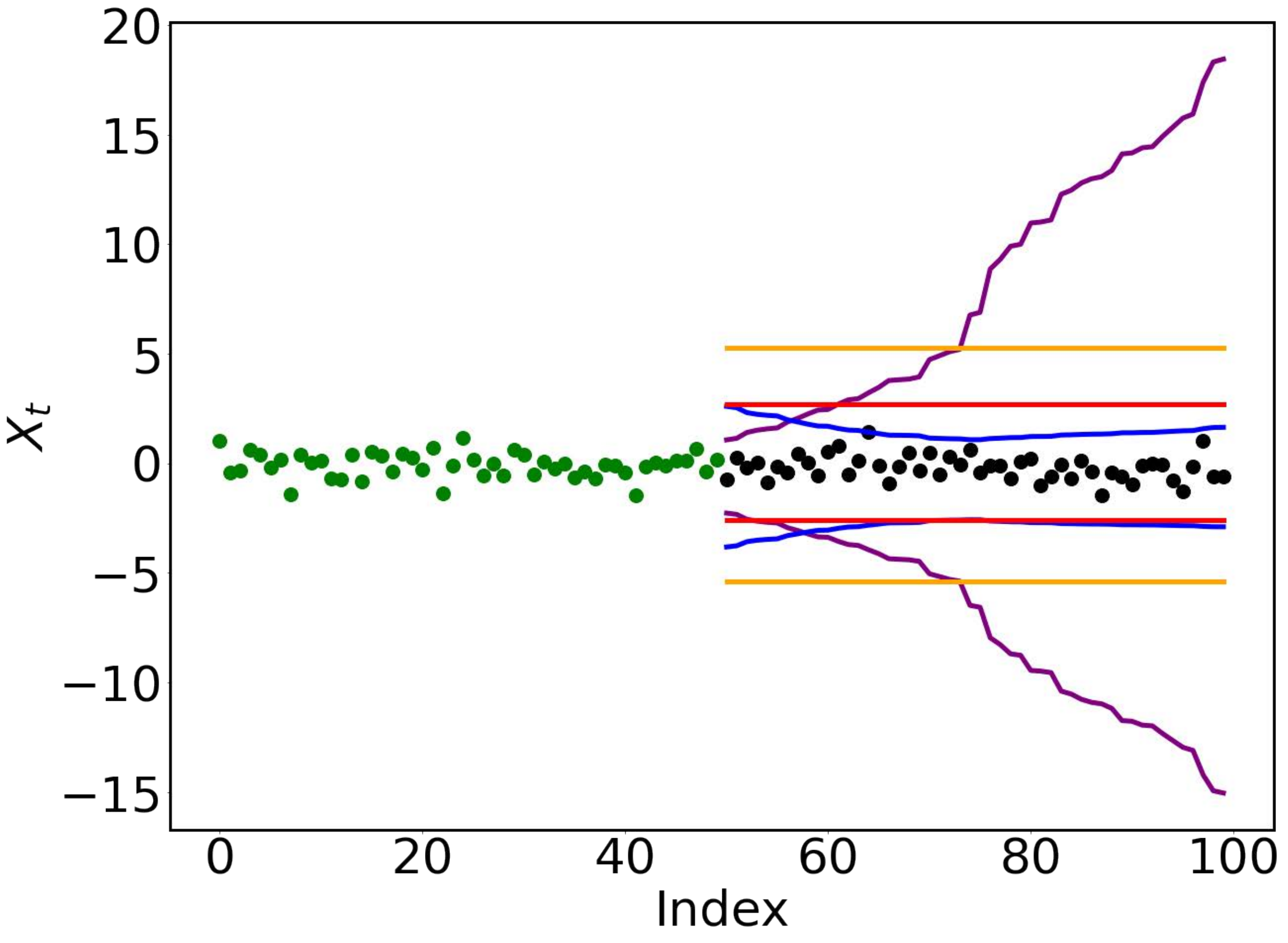}
        \centerline{(a) Synthetic}
    \end{minipage}
    \hfill
    \begin{minipage}[b]{0.15\textwidth}
        \includegraphics[width=\linewidth]{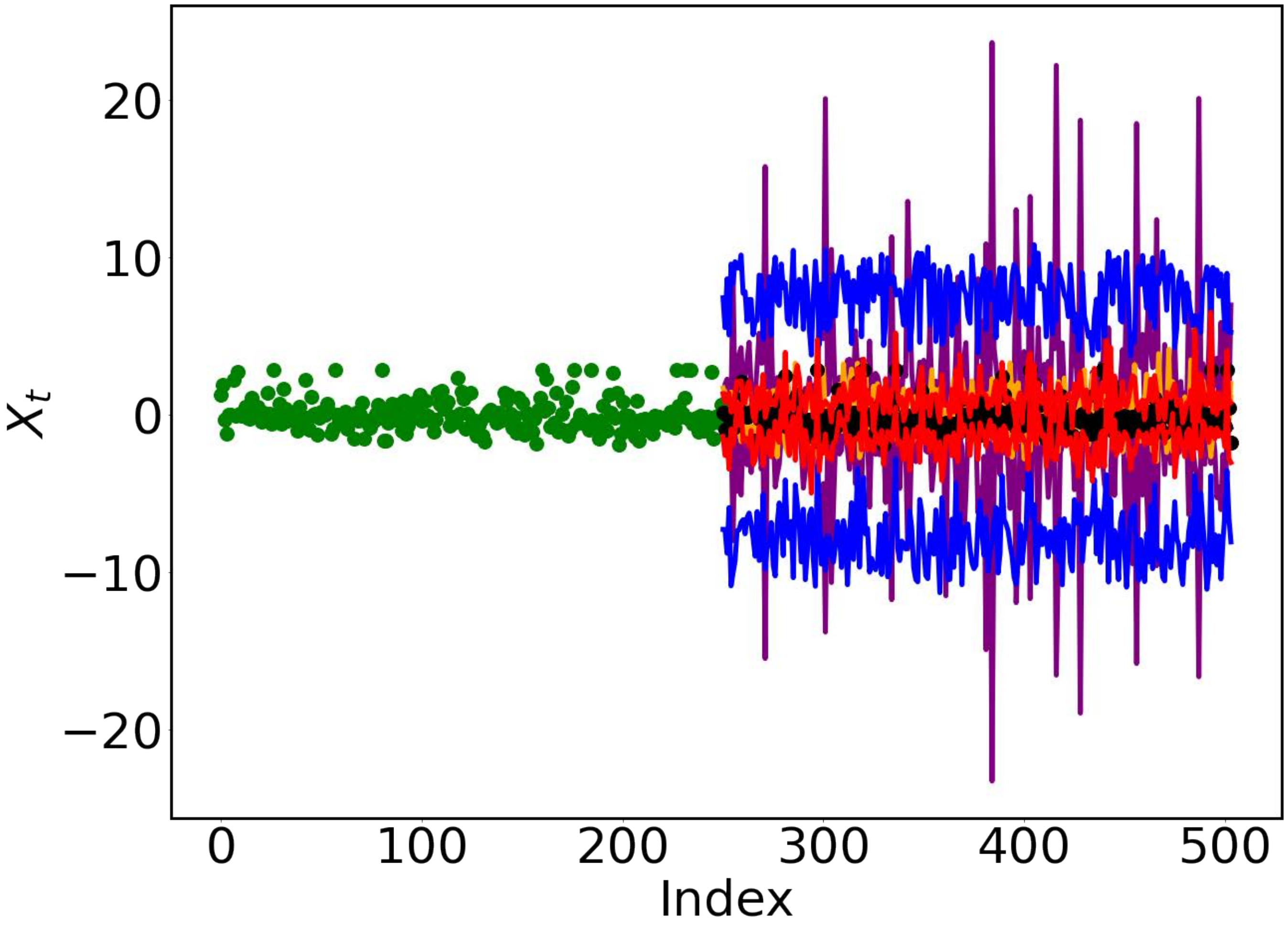}
        \centerline{(b) Boston House Price}
    \end{minipage}
    \hfill
    \begin{minipage}[b]{0.15\textwidth}
        \includegraphics[width=\linewidth]{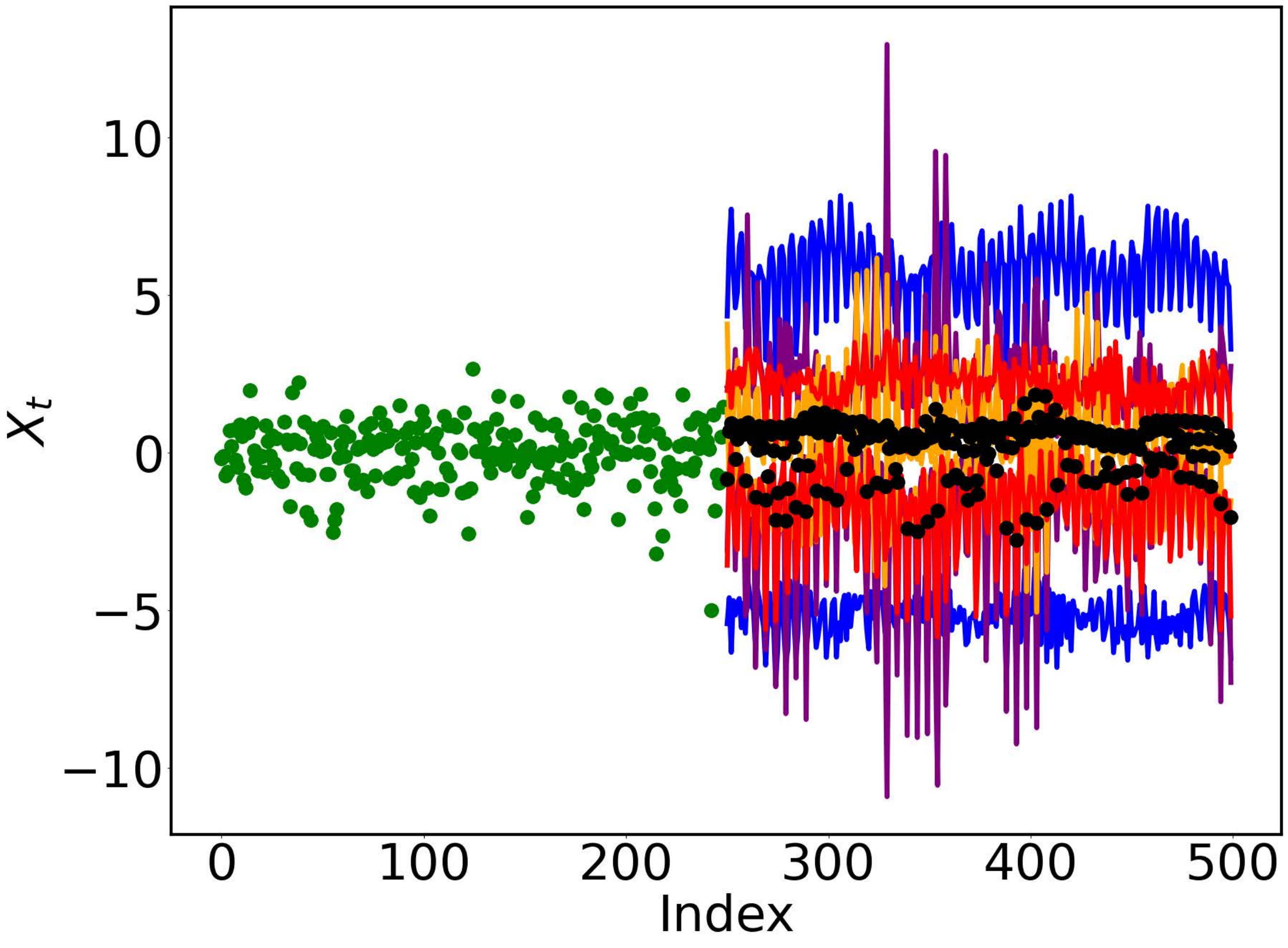}
        \centerline{(c) Sarcos}
    \end{minipage}
}
\fbox{%
    \begin{minipage}[b]{0.15\textwidth}
        \includegraphics[width=\linewidth]{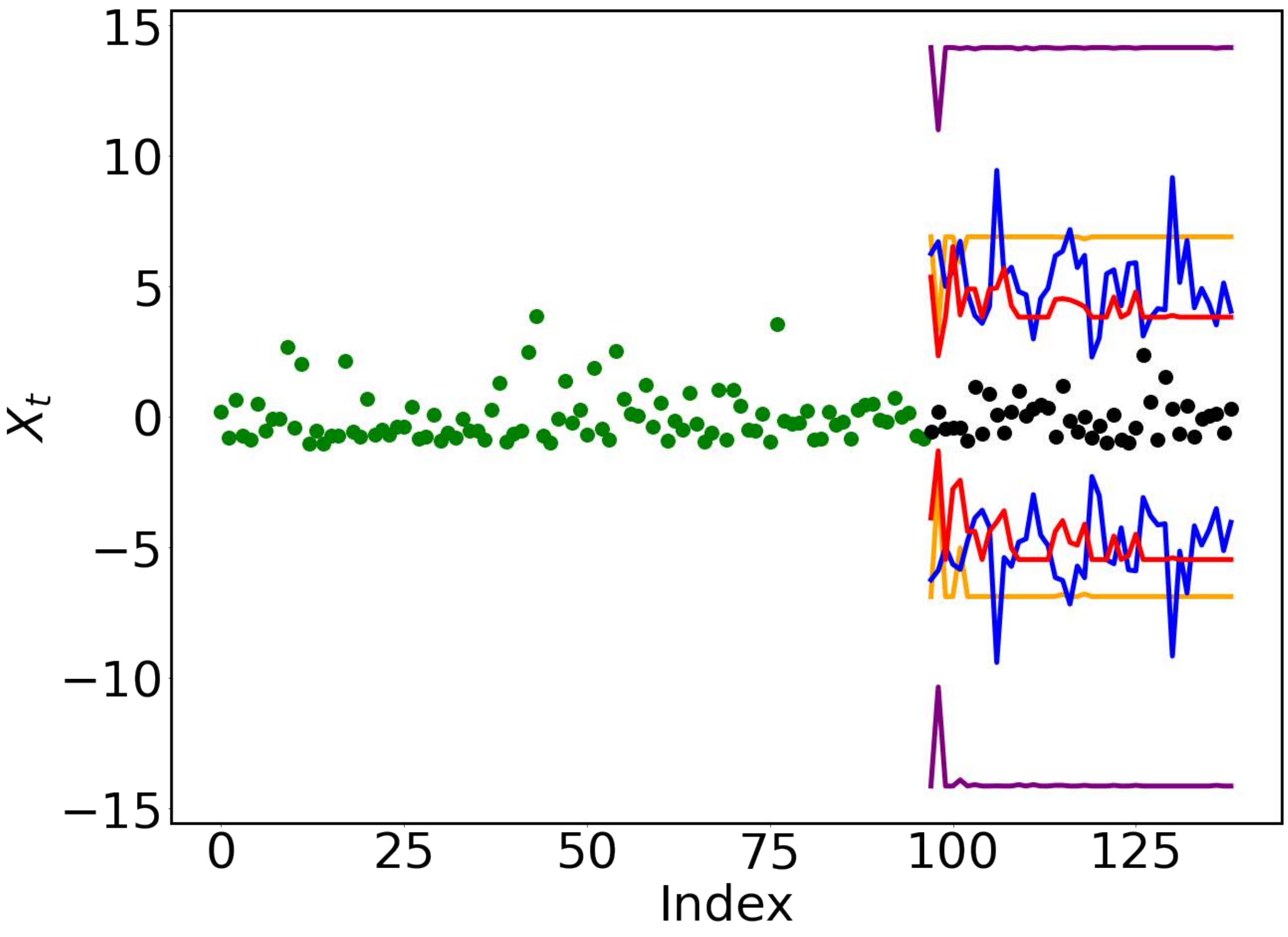}
        \centerline{(d) USGS Eq}
    \end{minipage}
    \hfill
    \begin{minipage}[b]{0.15\textwidth}
        \includegraphics[width=\linewidth]{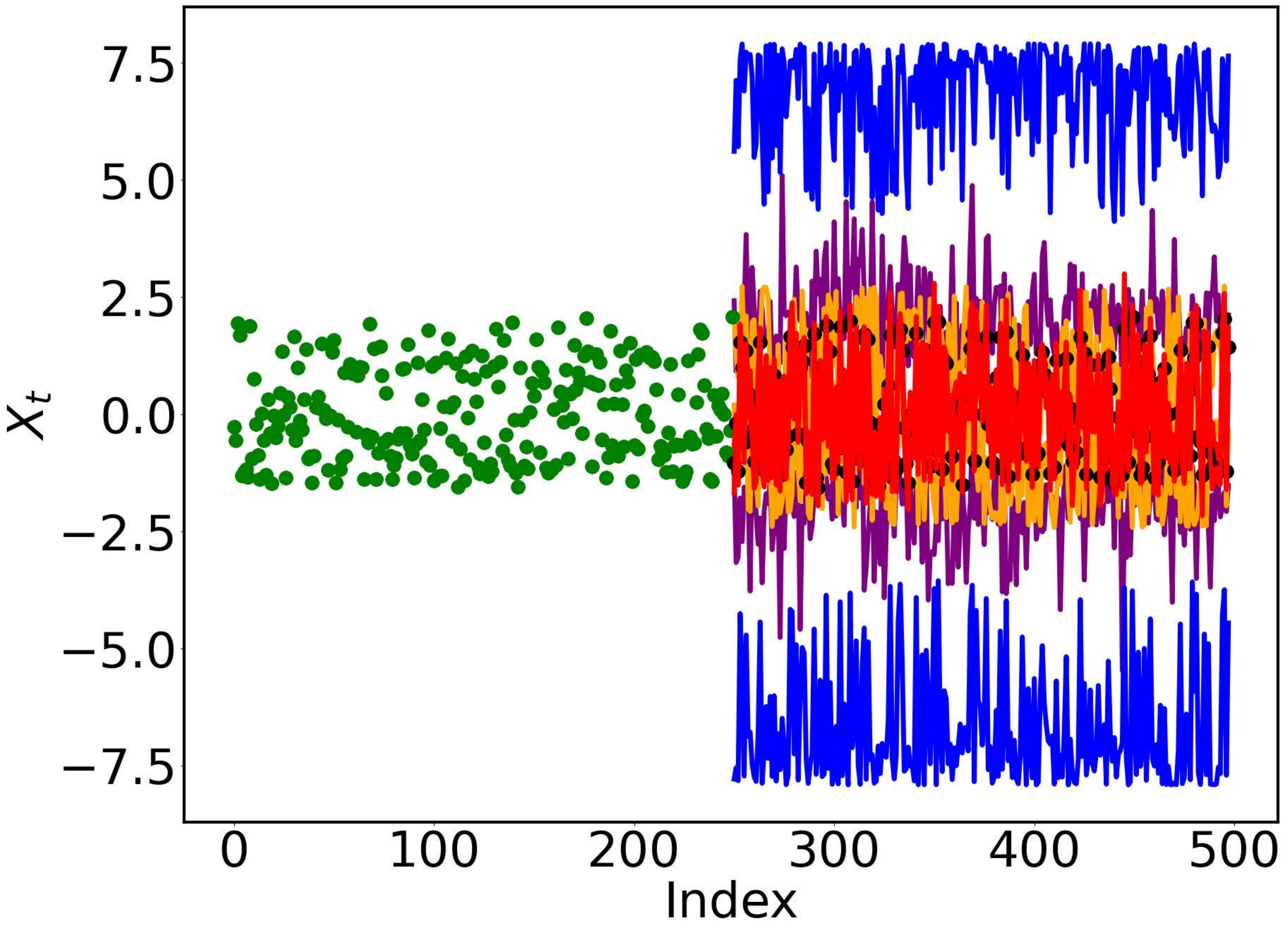}
        \centerline{(e) Loa\_CO2}
    \end{minipage}
    \hfill
    \begin{minipage}[b]{0.15\textwidth}
        \includegraphics[width=\linewidth]{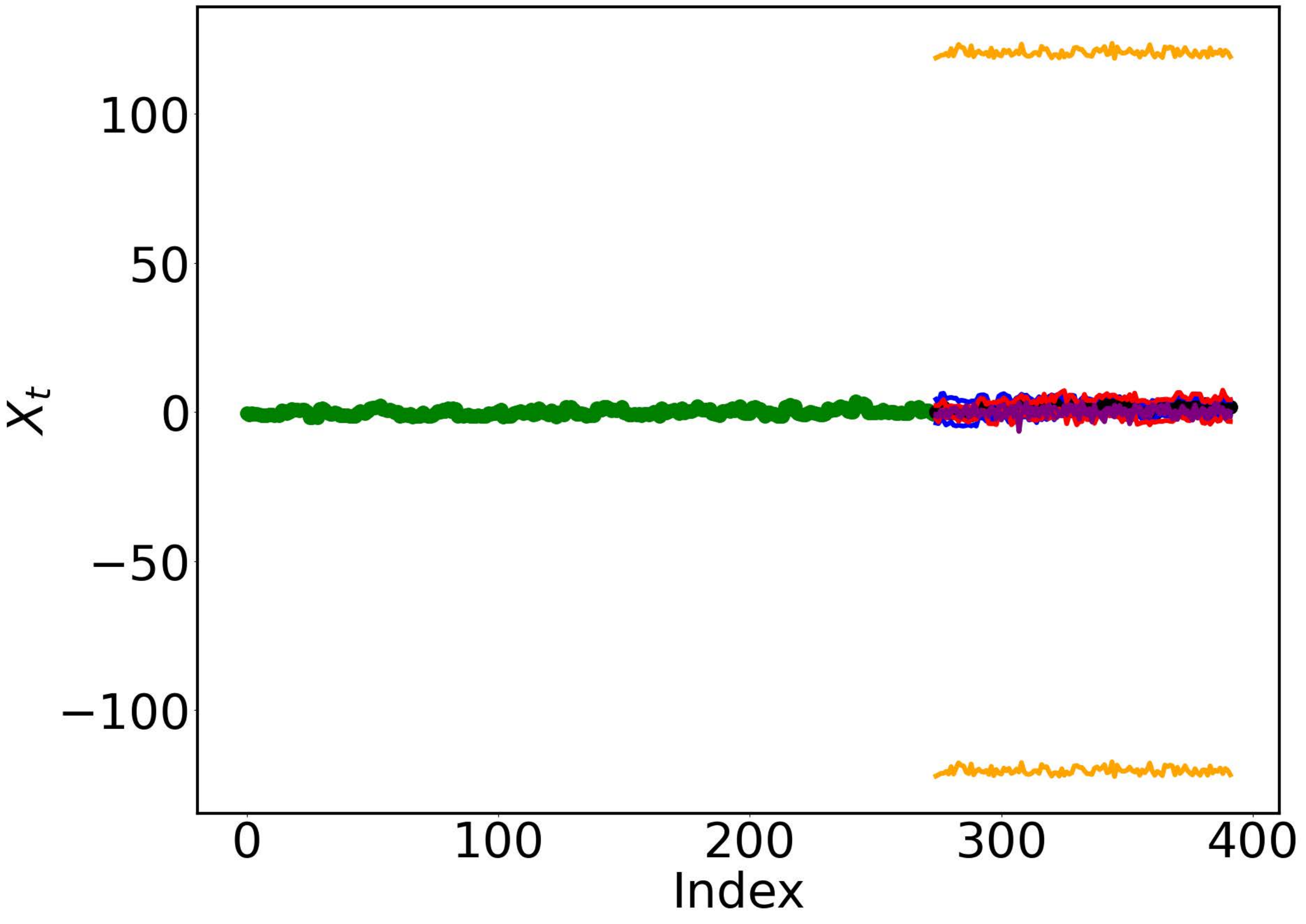}
        \centerline{(f) Auto-mpg}
    \end{minipage}
}
\end{center}
\caption{Comparison of our method with baselines for the test-point-specific bounds. The training set is in green, the test set in black, Lederer19 in orange, Fiedler21 in blue, Capone22 in purple, and our method in red.}
\label{figure:comparison}
\end{figure}
We illustrate the 99\% uncertainty bounds in Figure~\ref{figure:comparison} (large images are provided in Appendix~C.4), which corresponds to the results shown in Table~\ref{tab:confidence_comparison}. In all figures, our method achieves over 99\% coverage (with all black test points within bounds) and consistently produces tighter intervals, indicating its superior performance compared to all baselines. Computational cost results are reported in Appendix~C.1. Our method achieves competitive computational efficiency and strong scalability across datasets. 
Statistical significance results can be found in Appendix~C.5. In the vast majority of cases, our method demonstrates statistically significant improvements.

\section{Conclusion}

Our work addresses the limitations of existing methods by introducing a novel chaining-based approach that improves error control and robustness. Leveraging Talagrand’s techniques~\citep{talagrand2014upper}, 
we derive more flexible and accurate prediction bounds without relying on posterior means or variances. Our method not only yields conventional uncertainty bounds but also estimates the expected value and variance of the supremum using only the training set. Furthermore, it provides tighter bounds for commonly used kernels such as RBF and Matérn. 
Future work includes case analyses, point-selection strategies, and broader related work.

\section{Acknowledgments}
We thank all anonymous reviewers for their valuable feedback and constructive suggestions.

\bigskip

\bibliography{main}

\newpage
\onecolumn
\appendix
\setcounter{section}{0}                
\renewcommand\thesection{\Alph{section}}

\newcommand{\appsection}[1]{%
  \refstepcounter{section}
  \section*{Appendix \thesection.\ #1}
}

\renewcommand\thesubsection{\thesection.\arabic{subsection}}
\newcommand{\appsubsection}[1]{%
  \refstepcounter{subsection}%
  \subsection*{\thesubsection\ #1}%
}

\appsection{Chaining} 

\appsubsection{Review of Chaining} \label{A1}

Next we review at a high level the scheme of the chaining bound method as given by Talagrand.

 The goal is to bound \( \mathbb{E}Y \) where \( Y = \sup_t (X_t - X_{t_0}) \). We introduce a "good set" \( \Omega_u \) for a given parameter \( u \geq 0 \), which excludes undesirable events. As \( u \) becomes large, \( P(\Omega_u^c) \) becomes small. When \( \Omega_u \) occurs, we bound \( Y \), say \( Y \leq f(u) \), where \( f \) is an increasing function on \( \mathbb{R}_+ \).
\[
\mathbb{E}Y = \int_0^\infty P(Y \geq u)du \leq f(0) + \int_0^\infty P(Y \geq f(u))du,
\]
\[
\mathbb{E}Y = f(0) + \int_0^\infty f'(u) P(Y \geq f(u)) du,
\]
where we have used a change of variables in the last equality. Now, since \( Y \leq f(u) \) on \( \Omega_u \), we have:
\[
P(Y \geq f(u)) \leq P(\Omega_u^c),
\]
and finally:
\[
\mathbb{E}Y \leq f(0) + \int_0^\infty f'(u)P(\Omega_u^c)du.
\]
In practice, we will always have \( P(\Omega_u^c) \leq L \exp(-u/L) \) and \( f(u) = A + u^\alpha B \), yielding the bound:
\[
\mathbb{E}Y \leq A + K(\alpha)B.
\]

At the heart of this example is the introduction of a ``good set" \( \Omega_u \), which confines undesirable events to a small probability. As the parameter \( u \) increases, the probability of bad events \( \Omega_u^c \) decreases exponentially. This allows for effective error control within the ``good set," avoiding the coarse global error estimates typically used in traditional methods.

Furthermore, by controlling tail probabilities and utilizing exponential decay bounds, such as \( P(\Omega_u^c) \leq L \exp(-u/L) \), along with the function \( f(u) = A + u^\alpha B \), the chaining method ensures that the final error remains well-controlled. This level of probabilistic precision, achieved by breaking the problem into layers and managing each incremental error independently, prevents the overestimation of total error that is common in traditional approaches.

\appsubsection{An Example of Chaining} \label{A2}

\begin{figure}[ht]
\vskip 0.2in
\begin{center}
\centerline{\includegraphics[width=0.8\columnwidth]{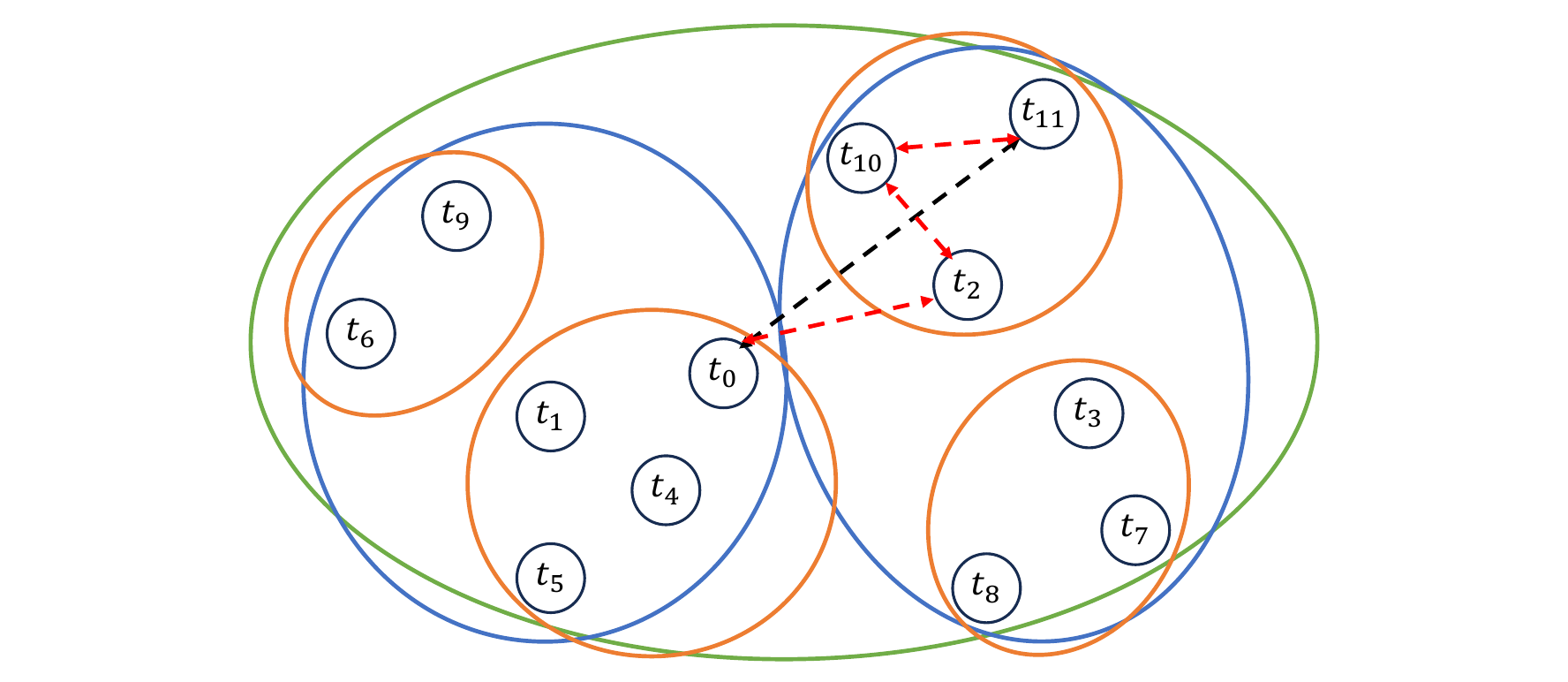}}
\caption{An example of why chaining helps with supremum.}
\label{example2}
\end{center}
\vskip -0.2in
\end{figure}

We illustrate the chaining method with a detailed example, as shown in the figure. The initial set \( T_0 = \{t_0\} \) maps all points to \( t_0 \) under \( \pi_0(t) \), such that \( \pi_0(t) = t_0 \) for all \( t \in T \), represented by the green region. The first layer \( T_1 = \{t_1, t_2\} \) maps points in the blue regions to either \( t_1 \) or \( t_2 \) via \( \pi_1(t) \), such that:
\[
\pi_1(t) = 
\begin{cases}
t_1, & \text{if } t \in \{t_0, t_1, t_4, t_5, t_6, t_9\}, \\
t_2, & \text{if } t \in \{t_2, t_3, t_7, t_8, t_{10}, t_{11}\}.
\end{cases}
\]
The second layer \( T_2 = \{t_4, t_6, t_7, t_{10}\} \) further refines the grouping, such that:
\[
\pi_2(t) = 
\begin{cases}
t_4, & \text{if } t \in \{t_0, t_1, t_4, t_5\}, \\
t_6, & \text{if } t \in \{t_6, t_9\}, \\
t_{10}, & \text{if } t \in \{t_2, t_{10}, t_{11}\}, \\
t_7, & \text{if } t \in \{t_3, t_7, t_8\}.
\end{cases}
\]
For example, considering \( t_{11} \), its mappings are \( \pi_0(t_{11}) = t_0 \), \( \pi_1(t_{11}) = t_2 \), and \( \pi_2(t_{11}) = t_{10} \). The decomposition of \( X_{t_{11}} - X_{t_0} \) is given by:
\[
X_{t_{11}} - X_{t_0} = (X_{t_{11}} - X_{\pi_2(t_{11})}) + (X_{\pi_2(t_{11})} - X_{\pi_1(t_{11})}) + (X_{\pi_1(t_{11})} - X_{t_0}),
\]
which simplifies to:
\[
X_{t_{11}} - X_{t_0} = (X_{t_{11}} - X_{t_{10}}) + (X_{t_{10}} - X_{t_2}) + (X_{t_2} - X_{t_0}).
\]
Each term \( X_t - X_{\pi_1(t)} \), \( X_{\pi_1(t)} - X_{\pi_2(t)} \), etc., is smaller compared to \( X_t - X_{t_0} \), making the supremum easier to compute. Furthermore, the grouping reduces the number of representative points, significantly lowering computational complexity.

\appsection{Proof of Theorems}

\appsubsection{Proof of Theorem 1} \label{proof1}

We provide a modified, more compact version to aid in exposition and intuition building. For the complete proof, please refer to Talagrand.

Assume that \((X_t)_{t \in T}\) is a Gaussian process, where each \(X_t\) is normally distributed with mean zero. For any two points \(s, t \in T\), the increment \(X_s - X_t\) is given by:
\[
E[(X_s - X_t)^2] = d(s, t)^2,
\]
where \(d(s, t)\) is a distance metric on \(T\).

Given a normally distributed random variable \(Z\) with mean zero and variance \(\sigma^2\), the probability that \(|Z|\) exceeds a threshold \(u\) is bounded by:
$
P(|Z| \geq u) \leq 2 \exp\left( -\frac{u^2}{2\sigma^2} \right)
$.
Applying this result to the increment \(X_s - X_t\), we substitute \(\sigma^2\) with \(d(s, t)^2\) and get:
\begin{align*}
    P(|X_s - X_t| \geq u) \leq 2 \exp\left( -\frac{u^2}{2d(s, t)^2} \right).
\end{align*}
This implies the expresion below when $u = u 2^{n/2} d(\pi_n(t), \pi_{n-1}(t))) $:
\begin{align*}
\text{P}(|X_{\pi_n(t)} - X_{\pi_{n-1}(t)}| 
\geq u 2^{n/2} d(\pi_n(t), \pi_{n-1}(t))) \leq  2 \exp \left( -u^2 2^{n-1} \right)  \notag
\end{align*}
The number of possible pairs \( (\pi_n(t), \pi_{n-1}(t)) \) is bounded by:
\[
|T_n| \cdot |T_{n-1}| \leq N_n N_{n-1} \leq N_{n+1} = 2^{2^{n+1}}.
\]
We define the (favorable) event $\Omega_{u,n}$ by
\begin{align}
\forall t, \, |X_{\pi_n(t)} - X_{\pi_{n-1}(t)}| \leq u2^{n/2}d(\pi_n(t), \pi_{n-1}(t)), \notag
\end{align}
and we define $\Omega_u = \bigcap_{n \geq 1} \Omega_{u,n}$. Then
\begin{align*}
p(u) := P(\Omega_u^c)  \notag
\leq \sum_{n \geq 1} P(\Omega_{u,n}^c) \leq \sum_{n \geq 1} 2 \cdot 2^{2^{n+1}} \exp(-u^2 2^{n-1}). \notag
\end{align*}
Here again, at the crucial step, we have used the union bound $P(\Omega_u^c) \leq \sum_{n \geq 1}P(\Omega_{u,n}^c)$. When $\Omega_u$ occurs, it yields
\begin{align}
|X_t - X_{t_0}| \leq u \sum_{n \geq 1} 2^{n/2} d(\pi_n(t), \pi_{n-1}(t)), \notag
\end{align}
so that
\begin{align}
\sup_{t \in T} |X_t - X_{t_0}| \leq uS,  \notag
\end{align}
where
\begin{align}
S := \sup_{t \in T} \sum_{n \geq 1} 2^{n/2} d(\pi_n(t), \pi_{n-1}(t)). \notag
\end{align}
Thus,
\begin{align}
P\left( \sup_{t \in T} |X_t - X_{t_0}| > uS \right) \leq p(u). \notag
\end{align}
Given \( n \geq 1 \) and \( u \geq 3 \), the series can be bounded by
\begin{align}
u^2 2^{n-1} \geq \frac{u^2}{2} + u^2 2^{n-2} \geq \frac{u^2}{2} + 2^{n+1}. \notag 
\end{align}

For
\begin{align*}
 p(u)
\leq  L \exp\left( -\frac{u^2}{2} \right),
\end{align*}
we observe that since \( p(u) \leq 1 \), the inequality holds not only for \(u \geq 3\) but also for \( u > 0 \), because \( 1 \leq \exp(\frac{9}{2})  \exp\left( -\frac{u^2}{2} - 2^{n+1} \right) \) for \(u \leq 3\). 
Hence,
\begin{align*}
P \left( \sup_{t \in T} |X_t - X_{t_0}| \geq uS \right) \leq L \exp\left( -\frac{u^2}{2} \right) 
\end{align*}
where $L$ is an constant term.  $\square$ 

\appsubsection{Proof of Theorem 2} \label{proof2}

Given any $t_0$ in $T$, the centering hypothesis implies
\begin{align*}
    E \sup_{t \in T} X_t &= E \sup_{t \in T} (X_t - X_{t_0}).
\end{align*}
The latter form has the advantage that we now seek estimates for the expectation of the nonnegative random variable \( Y = \sup_{t \in T} (X_t - X_{t_0}) \). For such a variable, we have the formula
\begin{align*}
    E Y &= \int_0^\infty P(Y \geq u) \, du. 
\end{align*}
Using Theorem 1:
\begin{align*}
P \left( \sup_{t \in T} |X_t - X_{t_0}| \geq uS \right) \leq L \exp\left( -\frac{u^2}{2} \right) 
\end{align*}

Since
\begin{align*}
\mathbb{E} \sup_{t \in T} X_t \leq \mathbb{E}\left[ X_{t_0}\right] + \mathbb{E}\left[ \sup_{t \in T} |X_t - X_{t_0}| \right] = X_{t_0} + \mathbb{E}\left[ \sup_{t \in T} |X_t - X_{t_0}| \right]
\end{align*}
so that
\begin{align*}
\mathbb{E} \sup_{t \in T} X_t \leq  X_{t_0} + \mathbb{E}\left[ \sup_{t \in T} |X_t - X_{t_0}| \right] \leq 
X_{t_0} + 
S \int_{0}^{\infty} P\left( \sup_{t \in T} |X_t - X_{t_0}| > uS \right) du.
\end{align*}
where $d(t, T_n))=\inf_{s \in T_n} \sqrt {K(t, t) + K(s, s) - 2K(t, s)} $ .

From it, to perform the integration, we introduce a new variable \(v\). Let \(v = \frac{u}{S}\), then \(du = S dv\). Thus,
\[
\mathbb{E}\left[ \sup_{t \in T} |X_t - X_{t_0}| \right] \leq L \cdot 
\int_{0}^{\infty} \exp\left(-\frac{v^2}{2}\right) S dv.
\]
Simplifying, we get:
\[
\mathbb{E}\left[ \sup_{t \in T} |X_t - X_{t_0}| \right] \leq LS \int_{0}^{\infty} \exp\left(-\frac{v^2}{2}\right) dv,
\]
where
\begin{align}
S := \sup_{t \in T} \sum_{n \geq 1} 2^{n/2} d(\pi_n(t), \pi_{n-1}(t)). \notag
\end{align}
This integral is a standard Gaussian integral, and the result is:
\[
\int_{0}^{\infty} \exp\left(-\frac{v^2}{2}\right) dv = \sqrt{\frac{\pi}{2}}.
\]
Since $\pi_n(t)$ approximates $t$, it is natural to assume that:
\begin{align}
d(t, \pi_n(t)) = d(t, T_n) := \inf_{s \in T_n} d(t, s). \notag
\end{align}
The triangle inequality yields:
\begin{align*}
d(\pi_n(t), \pi_{n-1}(t)) \leq d(t, \pi_n(t)) + d(t, \pi_{n-1}(t)) = d(t, T_n) + d(t, T_{n-1}),
\end{align*}
Making the change of variable \(n \leftarrow n + 1\) in the second sum below, we obtain:
\begin{align*}
S & = \sup_{t \in T} \sum_{n \geq 1} 2^{n/2} d(\pi_n(t), \pi_{n-1}(t)) \\ \notag
&\leq \sup_{t \in T} \sum_{n \geq 1} 2^{n/2} d(t, T_{n}) + \sup_{t \in T} \sum_{n \geq 1} 2^{n/2} d(t, T_{n-1}) \\ \notag
&= \sup_{t \in T} \sum_{n \geq 0} 2^{n/2} d'(t, T_n) + \sqrt{2} \sup_{t \in T} \sum_{n \geq 1}  2^{(n-1)/2} d(t, T_{n-1}) \\ \notag
&= \sup_{t \in T} \sum_{n \geq 0} 2^{n/2} d'(t, T_n) + \sqrt{2} \sup_{t \in T} \sum_{n \geq 0} 2^{n/2} d(t, T_n) \\ \notag
&\leq (1 + \sqrt{2}) \sup_{t \in T} \sum_{n \geq 0}  2^{n/2} d(t, T_n).
\end{align*}
Thus, the result is:
\[
\mathbb{E}\left[ \sup_{t \in T} |X_t - X_{t_0}| \right] \leq (1 + \sqrt{2}) L \sup_{t \in T} \sum_{n \geq 0} 2^{n/2}  d(t, T_n).
\] 
\begin{align*}
\mathbb{E} \sup_{t \in T} X_t \leq  X_{t_0} + \mathbb{E}\left[ \sup_{t \in T} |X_t - X_{t_0}| \right] \leq 
X_{t_0} + 
 (1 + \sqrt{2}) L \sup_{t \in T} \sum_{n \geq 0} 2^{n/2}  d(t, T_n). \quad\quad\square
\end{align*}

\appsubsection{Proof of Theorem 3}\label{proof3}

A common kernel used in GPR is the radial basis function (RBF) kernel, also known as the Gaussian kernel. In this context, we consider a composite kernel that combines a constant kernel with an RBF kernel. 
The constant kernel \( \sigma^2 \) adds a constant variance to the covariance matrix, helping to control the overall amplitude of the process. 
The combined kernel function is expressed as:
\[
K(s, t) = \sigma^2  \exp\left( -\frac{\|s - t\|^2}{2l^2} \right).
\]
By substituting \(K(s, s) = K(t, t) = 1\) and the kernel function \(K(s, t)\) into the distance formula, we obtain:
\[
d(s, t)^2 =  2 \sigma^2 (1 -  \exp\left(-\frac{\|s - t\|^2}{2l^2}\right) ).
\]
Using the Cauchy-Schwarz inequality In two-dimensional space, we get:  
\begin{align*}
    \frac{\|s - t\|^2 + \|t - u\|^2}{2} \geq \left( \frac{\|s - t\| + \|t - u\|}{2} \right)^2.
\end{align*}
Combined with the triangle inequality $\|s - t\| + \|t - u\| \geq \|s - u\|$, we then obtain: 
\[
\|s - t\|^2 + \|t - u\|^2  \geq \frac{\|s - u\|^2}{2}.
\]
Thus the distance is:
\begin{align*}
d(s, u)^2  
\leq 2 \sigma^2  \left(1 - \exp\left(-\frac{\|s - t\|^2 + \|t - u\|^2}{l^2}\right)\right). \\ \notag 
\end{align*}
Recall that the Taylor series expansion of \(\exp(x)\) is:
   \[
   \exp(x) = 1 + x + \frac{x^2}{2!} + \frac{x^3}{3!} + \cdots.
   \]
Let $x_{1} =-\frac{\|s - t\|^2}{l^2}$ and $x_{2} =-\frac{\|t - u\|^2}{l^2}$. We then get:
   \begin{align*}
   \exp(x_{1}) + \exp(x_{2}) -1  
   = 1 + (x_{1} + x_{2}) + \frac{x_{1}^2 + x_{2}^2}{2!} + \frac{x_{1}^3 + x_{2}^3}{3!} + \cdots \\ \notag
   \leq 1 + (x_{1}  + x_{2}) + \frac{(x_{1} + x_{2})^2}{2!} + \frac{(x_{1} + x_{2})^3}{3!} + \cdots 
   = \exp(x_{1} + x_{2}).
   \end{align*}
For this inequality, we provide another simpler proof: Given that \( x_1, x_2 \geq 0 \), it follows that \( \exp(x_1) \geq 1 \) and \( \exp(x_2) \geq 1 \). Therefore, \( (1 - \exp(x_1))(1 - \exp(x_2)) \geq 0 \), i.e., \( 1- \exp(x_{1}) - \exp(x_{2}) + \exp(x_{1} + x_{2}) \geq 0  \).

By using this, we have:
\begin{align*}
d(s, u)^2 & = 2 \sigma^2 \left(1 - \exp\left(x_{1} + x_{2}\right)\right) \\
&\leq 2 \sigma^2  + 2 \sigma^2   (1 - \exp\left(x_{1} \right) - \exp\left(x_{2}\right) ) \\ \notag 
&= 2 \sigma^2 ( 2- \exp^{\frac{1}{2}} \left(-\frac{\|s - t\|^2}{2l^2}\right) - \exp ^{\frac{1}{2}} \left(-\frac{\|t - u\|^2}{l^2}\right) )  \\ \notag 
&=  4 \sigma^2  - 2 \sigma K^{\frac{1}{2}}(s, t) - 2 \sigma K^{\frac{1}{2}}(t, u) \\
& = 2 \sigma^2 - 2 \sigma K^{\frac{1}{2}}(s, t) + 2 \sigma^2  - 2 \sigma K^{\frac{1}{2}}(t, u) \\
& = d'(s, t)^2 + d'(t, u)^2. 
\end{align*}
where $d'(s, t)^2 = K(s, s) + K(t, t) - 2 \sigma K^{\frac{1}{2}}(s,t)$.

Since $\pi_n(t)$ approximates $t$, it is natural to assume that:
\begin{align}
d(t, \pi_n(t)) = d(t, T_n) := \inf_{s \in T_n} d(t, s). \notag
\end{align}
For an RBF kernel, we have:
\begin{align*}
d(s, u)^2 & \leq  d'^2(s, t) + d'^2(t, u),
\end{align*}
where $d'(s, t)^2 = K(s, s) + K(t, t) - 2 \sigma K^{\frac{1}{2}}(s,t)$.

Making the change of variable \(n \leftarrow n + 1\) in the second sum below, we obtain:
\begin{align*}
S & = \sup_{t \in T} \sum_{n \geq 1} 2^{n/2} d(\pi_n(t), \pi_{n-1}(t)) \\ \notag
&\leq \sup_{t \in T} \sum_{n \geq 1} 2^{n/2} d'(t, T_{n}) + \sup_{t \in T} \sum_{n \geq 1} 2^{n/2} d'(t, T_{n-1}) \\ \notag
&= \sup_{t \in T} \sum_{n \geq 0} 2^{n/2} d'(t, T_n) + \sqrt{2} \sup_{t \in T} \sum_{n \geq 1}  2^{(n-1)/2} d'(t, T_{n-1}) \\ \notag
&= \sup_{t \in T} \sum_{n \geq 0} 2^{n/2} d'(t, T_n) + \sqrt{2} \sup_{t \in T} \sum_{n \geq 0} 2^{n/2} d'(t, T_n) \\ \notag
&\leq (1 + \sqrt{2}) \sup_{t \in T} \sum_{n \geq 0}  2^{n/2} d'(t, T_n).
\end{align*}
Using Theorem~2, we obtain the fundamental bound:
\begin{align*}
\mathbb{E} \sup_{t \in T} |X_t - X_{t_0}| \leq  
L \sup_{t \in T} \sum_{n \geq 0} 2^{n/2} d'(t, T_n),
\end{align*}
where 
\begin{align*}
d'(t, T_n)) = \inf_{s \in T_n} \sqrt {K(t, t) + K(s, s) - 2 \sigma K^{\frac{1}{2}}(t, s)}.  \quad\quad\square
\end{align*}

\appsubsection{Proof of Theorem~4}\label{proof4}

Since \(K(s, t) = \left(1 + \frac{\sqrt{3} \|s - t\|}{l} \right) \exp \left(-\frac{\sqrt{3} \|s - t\|}{l} \right) \), we have \(K(s, s) = K(t, t) = 1\).

By substituting \(K(s, s) = K(t, t) = 1\) and the kernel function \(K(s, t)\) into the distance formula, we obtain:
\[
d(s, t)^2 = 2 - 2\left(1 + \frac{\sqrt{3} \|s - t\|}{l} \right) \exp \left(-\frac{\sqrt{3} \|s - t\|}{l} \right).
\]

The Chebyshev's sum inequality is a fundamental result in the theory of inequalities. It states that if \(a_1, a_2\) and \(b_1, b_2\) are two sequences of real numbers that are sorted in opposite orders (one in increasing and the other in decreasing order), then the following inequality holds:
\begin{align*}
    \frac{1}{n} \sum_{i=1}^n a_i b_i \leq \left( \frac{1}{n} \sum_{i=1}^n a_i \right) \left( \frac{1}{n} \sum_{i=1}^n b_i \right).
\end{align*}
Specifically, for \( a_i = 1 + x_i \) and \( b_i = \exp(-x_i) \), which are oppositely sorted, let $x_1= \frac{\sqrt{3} \| s-t\|}{l} $ and $x_2= \frac{\sqrt{3} \| t-u\|}{l} $. Then the inequality for \(n=2\) becomes:
\begin{align*}
(1 + x_1)\exp(-x_1) + (1 + x_2)\exp(-x_2) 
\leq \frac{(1 + x_1 + 1 + x_2)[\exp(-x_1) + \exp(-x_2)]}{2}.
\end{align*}
Since $\| s-t\| \geq 0$ and $\| t-u\| \geq 0$, we have \( \exp(-x_i) \leq 1 \). Oberve that $\mbox{$(1 - \exp(-x_1))(1 - \exp(-x_2)) > 0$}$. Rearranging terms, we obtain:
\begin{align*}
 \exp(-x_1) + \exp(-x_2) 
< 1 + \exp(-x_1) \exp(-x_2) = 1 + \exp(-x_1-x_2).
\end{align*}
Using this, we get:
\begin{align*}
(1 + x_1)\exp(-x_1) + (1 + x_2)\exp(-x_2) &\leq \frac{(2 + x_1 + x_2)}{2} [\exp(-x_1) + \exp(-x_2) ] \\
&\leq \frac{(2 + x_1 + x_2)}{2} [1 + \exp(-x_1-x_2)]. \\
\end{align*}
After negating $\frac{(2 + x_1 + x_2)}{2} $, we get:
\begin{align*}
& (1 + x_1) [\exp(-x_1) - \frac{1}{2}] + (1 + x_2)[\exp(-x_2) - \frac{1}{2}] \leq (1 + x_1 + x_2) \exp(-x_1-x_2). \\
\end{align*}
Given the function $ f(x) = (1 + x) \exp(-x) $, the derivative of \( f(x) \) with respect to \( x \) is calculated using the product rule as:
\[ f'(x) = \frac{d}{dx} \left[ (1 + x) \exp(-x) \right] = - x \exp(-x). \]

Since $ \frac{\sqrt{3} \|s - t\|}{l} \geq 0$, we know that $f'(x) \leq 0$ when $x\geq 0 $. Thus $f(x)$is monotonically decreasing when $n\geq0$. 

With the triangle inequality $\|s - t\| + \|t - u\| \geq \|s - u\|$, and since $f(x)$ is monotonically decreasing, we get:
\begin{align*}
& K(s,u) = \left(1 + \frac{\sqrt{3} \| s-u\|}{l} \right) 
\exp \left(-\frac{\sqrt{3} \|s-u\|}{l} \right) \\
&\geq (1 + x_1 + x_2) \exp(-x_1-x_2) \\
&\geq (1 + x_1) [\exp(-x_1) - \frac{1}{2}] + (1 + x_2)[\exp(-x_2) - \frac{1}{2}] \\
& = K'(s,t) + K'(t,u),
\end{align*}
where $K'(s,t) = \left(1 + \frac{\sqrt{3} \| s-t\|}{l} \right) [\exp \left(-\frac{\sqrt{3} \|s-t\|}{l} \right) - \frac{1}{2}].\\ $
We can then calculate the distance:
\begin{align*}
d(s, u)^2  
&= K(s, s) + K(u, u) - 2K(s, u) \\
&\leq 2 - 2 [ K'(s,t) + K'(t,u) ] =  2 - 2K'(s,t) + 2 - 2K'(t,u) - 2 \\
&= d'(s, t)^2  + d'(t, u)^2  -2.
\end{align*}
For the Matérn kernel (with $v=\frac{3}{2}$), we have proven that:
\begin{align*}
d(s, u)^2 \leq  d'(s, t)^2 + d'(t, u)^2 -2,
\end{align*}
where $d'(s, t)^2 = K(s, s) + K(t, t) - 2K'(s,t)$.

Making the change of variable \(n \leftarrow n + 1\) in the second sum below, we get:
\begin{align*}
S &\leq \sup_{t \in T} \sum_{n \geq 1} 2^{n/2} \sqrt{d'^2(t, T_{n}) + d'^2(t, T_{n-1})  - 2 } \\
&\leq \sup_{t \in T} \sum_{n \geq 1} 2^{n/2} d'(t, T_n) + \sqrt{2} \sup_{t \in T} \sum_{n \geq 0} 2^{n/2} d'(t, T_n)- \sum_{n \geq 0} 2^{n/2} \sqrt{2}  \\
&\leq (1 + \sqrt{2})\sup_{t \in T} \sum_{n \geq 0}  2^{n/2}  ( d'(t, T_n) - \frac{\sqrt{2}}{1 + \sqrt{2}}).
\end{align*}
Using Theorem~2, we have the bound:
\begin{align*}
\mathbb{E} \sup_{t \in T} |X_t - X_{t_0}| \leq  
L  \sup_{t \in T} \sum_{n \geq 0} 2^{n/2} [d'(t, T_n)+\sqrt{2}-2],
\end{align*}
where $\mbox{$d'(t, T_n))=\inf_{s \in T_n} \sqrt {K(t, t) + K(s, s) - 2K'(t, s)}$}$, and 
$\mbox{$K'(s,t) = \left(1 + \frac{\sqrt{3} \| s-t\|}{l} \right) \left[\exp \left(-\frac{\sqrt{3} \|s-t\|}{l} \right) - \frac{1}{2}\right]$}$. $\quad\quad\square$

\appsubsection{Proof of Theorem~5}\label{proof5}

Given any $t_0 \in T$, the centering hypothesis implies
\[
\mathbb{E} \sup_{t \in T} X_t = \mathbb{E} \sup_{t \in T} (X_t - X_{t_0}) \leq X_{t_0} + \mathbb{E}[Y],
\]
where \( Y := \sup_{t \in T} |X_t - X_{t_0}| \) is a nonnegative random variable. By standard results,
\[
\mathbb{E}[Y] = \int_0^\infty \mathbb{P}(Y \geq u) \, du.
\]
Given a normally distributed random variable \(Z\) with mean zero and variance \(\sigma^2\), the probability that \(|Z|\) exceeds a threshold \(u\) is bounded by:
$
P(|Z| \geq u) \leq 2 \exp\left( -\frac{u^2}{2\sigma^2} \right)
$.
Applying this result to the increment \(X_s - X_t\), we substitute \(\sigma^2\) with \(d(s, t)^2\) and get:
\begin{align*}
    P(|X_s - X_t| \geq u) \leq 2 \exp\left( -\frac{u^2}{2d(s, t)^2} \right).
\end{align*}
This implies the expresion below when $u = u 2^{n/2} d(\pi_n(t), \pi_{n-1}(t))) $:
\begin{align*}
\text{P}(|X_{\pi_n(t)} - X_{\pi_{n-1}(t)}| 
\geq u 2^{n/2} d(\pi_n(t), \pi_{n-1}(t))) \leq  2 \exp \left( -u^2 2^{n-1} \right)  \notag
\end{align*}
The number of possible pairs \( (\pi_n(t), \pi_{n-1}(t)) \) is bounded by:
\[
|T_n| \cdot |T_{n-1}| \leq N_n N_{n-1} \leq N_{n+1} = 2^{2^{n+1}}.
\]
We define the (favorable) event $\Omega_{u,n}$ by
\begin{align}
\forall t, \, |X_{\pi_n(t)} - X_{\pi_{n-1}(t)}| \leq u2^{n/2}d(\pi_n(t), \pi_{n-1}(t)), \notag
\end{align}
and we define $\Omega_u = \bigcap_{n \geq 1} \Omega_{u,n}$. Then
\begin{align*}
p(u) := P(\Omega_u^c)  \notag
\leq \sum_{n \geq 1} P(\Omega_{u,n}^c) \leq \sum_{n \geq 1} 2 \cdot 2^{2^{n+1}} \exp(-u^2 2^{n-1}). \notag
\end{align*}
From the chaining construction, we have the high-probability bound:
\[
\mathbb{P}\left( \sup_{t \in T} |X_t - X_{t_0}| > u S \right) \leq p(u),
\]
with
\[
S := \sup_{t \in T} \sum_{n \geq 1} 2^{n/2} d(\pi_n(t), \pi_{n-1}(t)), \quad
p(u) := \sum_{n \geq 1} 2^{2^{n+1}+1} \exp(-u^2 2^{n-1}).
\]
Since \(p(u)\) is defined as the probability of the event \(\Omega_u^c\), it is bounded above by 1. However, the upper bound derived via union bound can exceed 1 when \(u\) is small, so we apply:
\[
\mathbb{E}[Y] \leq \int_0^\infty \min(p(u/S), 1) \, du.
\]
Letting \(v = u/S\), we change variable to get:
\[
\mathbb{E}[Y] \leq S \int_0^\infty \min(p(v), 1) \, dv.
\]
Thus, the final upper bound becomes:
\[
\mathbb{E} \sup_{t \in T} X_t \leq X_{t_0} + S \int_0^\infty \min\left( \sum_{n \geq 1} 2^{2^{n+1}+1} \exp(-v^2 2^{n-1}), 1 \right) dv. \quad\quad\square
\]
Proof of Theorem~3 and Proof of Theorem~4, we have $S \leq (1 + \sqrt{2}) \sup_{t \in T} \sum_{n \geq 0} 2^{n/2} inf_{s \in T_n} \sqrt {K(t, t) + K(s, s) - 2 \sigma K^{\frac{1}{2}}(t, s)} $ for RBF kernel; $S \leq (1 + \sqrt{2}) \sup_{t \in T} \sum_{n \geq 0} 2^{n/2} [\inf_{s \in T_n} \sqrt {K(t, t) + K(s, s) - 2K'(t, s)}+\sqrt{2}-2] $, and
$\mbox{$K'(s,t) = \left(1 + \frac{\sqrt{3} \| s-t\|}{l} \right) \left[\exp \left(-\frac{\sqrt{3} \|s-t\|}{l} \right) - \frac{1}{2}\right]$}$ for Matérn kernel.  

Compared to Talagrand's chaining bounds, our construction avoids introducing a potentially large prefactor \(L\) (e.g., \(L = \exp(9/2)\)) that arises in analytic derivations. Instead, we directly bound the tail integral numerically, truncating values exceeding 1. This is justified because \(p(v)\) corresponds to a probability, and is observed to be well below 1 for moderate \(v\), yielding a tighter and more practical bound.


\appsubsection{Proof of Theorem~6}\label{proof6}

Given a normally distributed random variable \(Z\) with mean zero and variance \(\sigma^2\), the probability that \(|Z|\) exceeds a threshold \(u\) is bounded by:
\[
P(|Z| \geq u) \leq 2 \exp\left( -\frac{u^2}{2\sigma^2} \right)
\]
Applying this result to the increment \(X_s - X_t\), we substitute \(\sigma^2\) with \(d(s, t)^2\) and get:
\[
    P(|X_s - X_t| \geq u) \leq 2 \exp\left( -\frac{u^2}{2d(s, t)^2} \right).
\]
This implies the expresion below when $u = u 2^{n/2} d(\pi_n(t), \pi_{n-1}(t))) $:
\begin{align*}
\text{P}(|X_{\pi_n(t)} - X_{\pi_{n-1}(t)}| 
\geq u 2^{n/2} d(\pi_n(t), \pi_{n-1}(t))) \\
\leq  2 \exp \left( -u^2 2^{n-1} \right)  
\end{align*}
 By introducing the confidence parameter 
\( \delta = 2 \exp(-u^2 2^{n-1}) \). Taking the natural logarithm of both sides yields
$
\ln(\delta) = \ln(2) - u^2 2^{n-1}.
$
Rearranging for \( u^2 \), we obtain
$
u = \sqrt{\frac{\ln(2) + \ln(1/\delta)}{2^{n-1}}}.
$

Substituting this expression for \( u \) back into the inequality for \( |X_s - X_t| \), we have:
\[
|X_s - X_t| \leq u 2^{n/2} d(s, t),
\]
Since $\pi_n(t)$ approximates $t$, it is natural to assume that:
\begin{align*}
d(t, \pi_n(t)) = d(t, T_n) := \inf_{s \in T_n} d(t, s). 
\end{align*} 
Using the expression for \( u \), the bound becomes
\[
|X_t - X_{\pi_n(t)}| \leq \sqrt{\frac{\ln(2) + \ln(1/\delta)}{2^{n-1}}} \cdot 2^{n/2} \cdot d(t, T_n).  \quad\quad\square
\]

\appsection{More Experimental Results and Analysis}
\label{app:expt}

\appsubsection{Computational Cost and Scalability}  \label{A20}

The proposed method has three primary computational steps: fitting the Gaussian process, constructing the sets \( \{T_n\}\), and computing bounds for the test points. Fitting the Gaussian process involves matrix factorization with a complexity of \( \mathcal{O}(|D_{\text{train}}|^3) \). Constructing \( \{T_n\} \) requires \( \mathcal{O}(|D_{\text{train}}|^2 \cdot \log \log |D_{\text{train}}|) \), dominated by kernel distance computations. Finally, computing bounds for \( |D_{\text{test}}| \) test points has a complexity of \( \mathcal{O}(|D_{\text{test}}| \cdot |D_{\text{train}}| \cdot \log \log |D_{\text{train}}|) \). The total computational complexity depends on the relative sizes of the training and test sets. Since the sizes of the training and test sets can vary, the overall complexity is determined by the more computationally intensive step. Thus, the total time complexity is:
\( \mathcal{O}(\max(|D_{\text{train}}|^2 \cdot \log \log |D_{\text{train}}|, |D_{\text{test}}| \cdot |D_{\text{train}}| \cdot \log \log |D_{\text{train}}|))\).

We also evaluated computational cost and scalability, with the table detailing data size and runtime for each method. All numerical experiments in this section were conducted on a Linux system with kernel version 5.15.0-112-generic (\#122-Ubuntu SMP Thu May 23 07:48:21 UTC 2024). The machine configuration includes an x86\_64 processor with 16 CPU cores and 125.49 GB of RAM. 

\begin{table}[h]
\centering
\begin{tabular}{c|cccccc}
\hline
 & Synthetic & Boston House  & Sarcos & USGS Earthquake & Loa CO2 & Autompg
 \\ \hline
Train Data Size & 50 & 250 & 250  & 97 & 250 & 274
\\ \hline
Test Data Size & 50 & 254 & 4000 & 42  & 248  & 118
\\ \hline
\end{tabular}
\caption{Size of datasets.}
\label{table:size}
\end{table}


\begin{table}[h]
\centering
\begin{tabular}{c|cccccc}
\hline
Time(s) & Synthetic& Boston House &  Sarcos & USGS Earthquake & Loa CO2 & Autompg \\ \hline
RBF(Ours) & 0.05 & \textbf{0.52} &  1.74 & 0.48 & 2.38 & \textbf{1.48} \\ \hline
Matérn(Ours) & \textbf{0.04} & 0.77  & 1.75 & 0.69 & 3.01 & 1.71 \\ \hline
Capone22 & 30.68 & 149.63  & 343.54 & 6.51 & 100.43 & 9.02 \\ \hline
Fiedler21 & 0.07 & 0.75 & \textbf{1.18} & \textbf{0.29} & \textbf{1.20} & 1.70\\ \hline
Lederer19 & 0.56 & 2.31 & 2.86 &  4.23 & 5.37 & 2.59\\ \hline
\end{tabular}
    \caption{Computational cost and scalability of our method with baselines in synthetic and real-world datasets.}
    \label{table:compute}
\end{table}

For the computational cost, our methods (RBF and Matérn) perform competitively across various datasets in Table~\ref{table:compute}. On smaller datasets such as Synthetic Data and Boston House Price, RBF and Matérn exhibit significantly lower runtimes than the baseline methods. For example, RBF requires only 0.05 seconds on Synthetic Data and 0.52 seconds on Boston House Price, while Matérn achieves the lowest runtime of 0.04 seconds on Synthetic Data. However, as the dataset size increases, Fiedler21 demonstrates a computational advantage, achieving 1.18 seconds on Sarcos, outperforming both RBF (1.74 seconds) and Matérn (1.75 seconds). Therefore, while our methods excel on smaller to medium-sized datasets, Fiedler21 shows improved efficiency on larger datasets like Sarcos.

Scalability was assessed by analyzing performance across increasingly large datasets. Both RBF and Matérn display strong scalability, maintaining relatively stable runtimes even with substantial increases in dataset size, such as on the Sarcos dataset. This consistent performance highlights their adaptability to larger datasets with minimal loss of efficiency. Fiedler21 also scales well, showing competitive performance on large datasets with a runtime of 1.18 seconds on Sarcos, making it suitable for large-scale applications. In contrast, Capone22’s runtime increases drastically with data size, reaching 343.54 seconds on Sarcos, indicating poor scalability and limiting its practicality for very large datasets. Lederer19 exhibits moderate scalability, performing well on medium-sized datasets but showing some inefficiencies as dataset size increases.

\appsubsection{Uncertainty Bound Prediction for Each Point in the Test Set} \label{A21}

Table \ref{table:comparison_90} compares the performance of our method with baseline methods in terms of CWC values under 99\%, 95\%, and 90\% confidence levels. CWC is the primary metric for evaluation as it balances the trade-off between coverage and interval width. In the vast majority of cases, our RBF and Matérn kernels achieved the lowest CWC values, outperforming the baseline methods. This demonstrates that our methods maintain superior coverage while providing more compact intervals. While there were a few instances where other methods performed comparably, our method generally provided more compact uncertainty bounds and proved effective in delivering accurate uncertainty estimates across varying confidence levels.

\appsubsection{Leveraging the Training Set to Compute Uncertainty Bounds over All Unseen Test Points} \label{A22}

We computed the upper and lower bounds using only the training set, without constructing probabilistic uncertainty intervals. For the baseline methods, which rely on posterior means and scaled posterior standard deviations, we modified their approach by performing predictions on the training set and directly determining the absolute extrema as the upper and lower bounds. These bounds were subsequently validated on the unseen test set.

We evaluated our method on several real-world datasets by assessing whether the computed bounds successfully covered all unseen test points, as presented in Table~\ref{table:comparison_sup_90}. The results clearly demonstrate that, under this evaluation protocol, variations across methods are minimal, with almost negligible differences when rounded to two decimal places. Across all datasets, our RBF and Matérn kernels consistently produced tighter bounds and achieved the lowest CWC values, outperforming the baseline methods.

Notably, our method maintained full coverage (PICP = 1.00) on all datasets while attaining smaller CWC values, thereby substantiating its effectiveness in providing precise and compact extrema estimates. This illustrates the practical utility of our approach in estimating bound ranges, establishing its relevance and effectiveness in real-world applications.

\begin{table*}[h]
    \centering
    \setlength{\tabcolsep}{3pt} 
    \begin{tabular}{l|ccc|ccc|ccc}
        \toprule
        99\% & PICP & NMPIW & \textbf{CWC(\( \downarrow \)}) 
               & PICP & NMPIW & \textbf{CWC(\( \downarrow \)}) 
               & PICP & NMPIW & \textbf{CWC(\( \downarrow \)})  \\ \cline{2-10}
        & \multicolumn{3}{c|}{Synthetic Data} & \multicolumn{3}{c|}{BH (Boston Housing)} & \multicolumn{3}{c}{Sarcos} \\ \cline{2-10}
                RBF (Ours) & 1.00 & 1.80 & 1.80 & 0.99 & 1.01 & 1.01 & 1.00 & 0.48 & 0.48 \\
        Matérn (Ours) & 1.00 & 1.80 & \textbf{1.80} & 0.99 & 0.76 & \textbf{0.76} & 1.00 & 0.40 & \textbf{0.40} \\
        Capone22 & 1.0 & 5.30 & 5.30 & 1.00 & 1.30 & 1.30 & 1.00 & 0.53 & 0.53 \\
        Fiedler21 & 0.98 & 1.49 & 3.95 & 1.00 & 3.46 & 3.46 & 1.00 & 1.42 & 1.42 \\
        Lederer19 & 1.00 & 3.63 & 3.63 & 1.00 & 1.47 & 1.47 & 1.00 & 0.56 & 0.56 \\
        CI & 0.95 & 5.34 & 37.19 & 
0.94 & 0.47 & 5.95 &
0.94 & 0.18 & 2.03 \\ 
        \hline
    & \multicolumn{3}{c|}{USGS Earthquake} & \multicolumn{3}{c|}{Loa\_CO2} & \multicolumn{3}{c}{Auto-mpg} \\ \cline{2-10}
RBF(Ours)       & 1.00   & 2.69   & \textbf{2.69}         & 1.00   & 0.81   & 0.81                   & 1.00   & 1.09      & 1.09                   \\
Matérn(Ours)    & 1.00   & 2.76   & 2.76                  & 1.00   & 0.24   & \textbf{0.24}          & 1.00   & 0.84      & \textbf{0.84}                   \\
Capone22        & 1.00   & 8.41   & 8.41                  & 0.84   & 1.12   & 1761.85                & 0.93   & 0.36      & 6.85                   \\
Fiedler21       & 1.00   & 3.26   & 3.26                  & 1.00   & 3.71   & 3.71                   & 1.00   & 1.39      & 1.39                   \\
Lederer19       & 1.00   & 4.23   & 4.23                  & 1.00   & 0.55   & 0.55                   & 1.00   & 50.03     & 50.03 
\\ 
Bayesian CI & 0.96 & 1.19 & 6.13 & 
0.94 & 0.55 & 7.31 & 
0.96 & 0.68 & 3.51
\\
\bottomrule
%
        95\% & PICP & NMPIW & CWC(\( \downarrow \)) 
               & PICP & NMPIW & CWC(\( \downarrow \)) 
               & PICP & NMPIW & CWC(\( \downarrow \))  \\ \cline{2-10}
        & \multicolumn{3}{c|}{Synthetic Data} & \multicolumn{3}{c|}{BH (Boston Housing)} & \multicolumn{3}{c}{Sarcos} \\ \cline{2-10}
      RBF (Ours) & 1.00 & 1.50 & 1.50 & 0.99 & 0.84 & 0.84 & 1.00 & 0.40 & 0.40 \\
        Matérn (Ours) & 1.00 & 1.50 & 1.50 & 0.97 & 0.64 & \textbf{0.64} & 0.99 & 0.33 & \textbf{0.33} \\
        Capone22 & 0.94 & 2.65 & 7.02 & 1.00 & 1.28 & 1.28 & 1.00 & 0.51 & 0.51 \\
        Fiedler21 & 0.98 & 1.49 & \textbf{1.49} & 1.00 & 3.46 & 3.46 & 1.00 & 1.42 & 1.42 \\
        Lederer19 & 1.00 & 3.56 & 3.56 & 1.00 & 1.35 & 1.35 & 1.00 & 0.50 & 0.50 \\
        Bayesian CI & 0.89 & 4.07 & 69.02 &
0.90 & 0.36 & 4.94 & 
0.89 & 0.14 & 3.44 \\
        \hline
& \multicolumn{3}{c|}{USGS Earthquake} & \multicolumn{3}{c|}{Loa\_CO2} & \multicolumn{3}{c}{Auto-mpg} \\ \cline{2-10}
RBF (Ours)      & 1.00   & 2.25   & 2.25                  & 1.00   & 0.68   & 0.68                   & 1.00   & 0.91   & 0.91                  \\
Matérn (Ours)   & 1.00   & 2.25   & \textbf{2.25}         & 1.00   & 0.20   & \textbf{0.20}          & 1.00   & 0.70   & \textbf{0.70}                   \\
Capone22        & 1.00   & 8.41   & 8.41                  & 0.84   & 1.12   & 239.41                 & 0.91   & 0.33   & 3.23                   \\
Fiedler21       & 1.00   & 3.26   & 3.26                  & 1.00   & 3.71   & 3.71                   & 1.00   & 1.39   & 1.39                   \\
Lederer19       & 1.00   & 4.13   & 4.13                  & 1.00   & 0.53   & 0.53                   & 1.00   & 48.42  & 48.42                   \\ 
Bayesian CI & 0.93 & 0.91 & 3.00 & 
0.91 & 0.42 & 3.30 & 
0.93 & 0.52 & 2.34 \\
\bottomrule
%
        90\% & PICP & NMPIW & CWC(\( \downarrow \)) 
               & PICP & NMPIW & CWC(\( \downarrow \)) 
               & PICP & NMPIW & CWC(\( \downarrow \))  \\ \cline{2-10}
        & \multicolumn{3}{c|}{Synthetic Data} & \multicolumn{3}{c|}{BH (Boston Housing)} & \multicolumn{3}{c}{Sarcos} \\ \cline{2-10}
RBF(Ours)       & 0.98   & 1.35   & 1.35                  & 0.98   & 0.76   & 0.76                   & 1.00   & 0.36   & 0.36                   \\
Matérn(Ours)    & 1.00   & 1.35   & \textbf{1.35}         & 0.96   & 0.57   & \textbf{0.57}          & 0.98   & 0.30   & \textbf{0.30}          \\
Capone22        & 0.94   & 2.63   & 2.63                  & 1.00   & 1.18   & 1.18                   & 1.00   & 0.50   & 0.50                   \\
Fiedler21       & 0.98   & 1.49   & 1.49                  & 1.00   & 3.46   & 3.46                   & 1.00   & 1.42   & 1.42                   \\
Lederer19       & 1.00   & 3.53   & 3.53                  & 1.00   & 1.41   & 1.41                   & 1.00   & 0.49   & 0.49                   \\ 
Bayesian CI & 0.83 & 3.41 & 113.08 & 
0.86 & 0.30 & 2.66 & 
0.83 & 0.12 & 3.22 \\
        \hline
& \multicolumn{3}{c|}{USGS Earthquake} & \multicolumn{3}{c|}{Loa\_CO2} & \multicolumn{3}{c}{Auto-mpg} \\ \cline{2-10}
RBF(Ours)       & 1.00   & 2.03   & 2.03                  & 1.00   & 0.61   & 0.61                   & 1.00   & 0.82   & 0.82                   \\
Matérn(Ours)    & 1.00   & 2.03   & \textbf{2.03}       & 1.00   & 0.18   & \textbf{0.18}          & 1.00   & 0.63   & \textbf{0.63}                   \\
Capone22        & 1.00   & 8.41   & 8.41                  & 0.84   & 1.12   & 20.68                  & 0.86   & 0.28   & 1.95                   \\
Fiedler21       & 1.00   & 3.26   & 3.26                  & 1.00   & 3.71   & 3.71                   & 1.00   & 1.39   & 1.39                   \\
Lederer19       & 1.00   & 4.07   & 4.07                  & 1.00   & 0.52   & 0.52                   & 1.00   & 47.70  & 47.70                   \\ 
Bayesian CI & 0.88 & 0.76 & 2.82 & 
0.88 & 0.35 & 1.54 & 
0.89 & 0.44 & 0.93 \\
\bottomrule
    \end{tabular}
        \caption{Comparison of our methods against baselines on synthetic and real-world datasets at 99\%, 95\%, 90\% confidence levels for the test-point-specific bounds.}
    \label{table:comparison_90}
\end{table*}

\appsubsection{Real-World Data}  \label{A23}

Traditional methods typically rely on the entire kernel function to compute the mean of the data points. While effective in many cases, this approach struggles to handle strong local correlations because it does not account for the varying relationships between data points in localized regions. In contrast, the chaining method is particularly well-suited for such scenarios. It groups data within highly correlated regions by defining successive approximation layers. Each layer progressively refines the approximation, improving the accuracy of the estimate while controlling the error. This approach allows the chaining method to capture local variations more effectively, making it a robust choice for handling complex data structures.

\begin{table*}[h]
    \centering
    \setlength{\tabcolsep}{3pt} 
    \renewcommand{\arraystretch}{0.9} 
    \begin{tabular}{l|ccc|ccc|ccc}
        \toprule
         & PICP & NMPIW & CWC(\( \downarrow \)) 
               & PICP & NMPIW & CWC(\( \downarrow \)) 
               & PICP & NMPIW & CWC(\( \downarrow \))  \\ \cline{2-10}
        & \multicolumn{3}{c|}{Synthetic Data} & \multicolumn{3}{c|}{BH (Boston Housing)} & \multicolumn{3}{c}{Sarcos} \\ \cline{2-10}
RBF(Ours)       & 1.00   & 1.68   & 1.68                  & 1.00   & 1.75  & 1.75                   & 1.00   & 1.03   & 1.03                   \\
Matérn(Ours)    & 1.00   & 1.67   & \textbf{1.67}                   & 1.00   & 1.64   & \textbf{1.64}          & 1.00   & 0.78   & \textbf{0.78}          \\
Capone22        & 0.96   & 0.89 & 4.89           & 1.00   & 1.77   & 1.77                   & 1.00   & 1.18   & 1.18                   \\
Fiedler21       & 1.00   & 2.20   & 2.20                  & 1.00   & 5.04   & 5.04                   & 1.00   & 2.31   & 2.31                   \\
Lederer19       & 1.00   & 2.02   & 2.02                  & 1.00   & 1.78   & 1.78                  & 1.00   & 1.34   & 1.34                   \\ \hline

& \multicolumn{3}{c|}{USGS Earthquake} & \multicolumn{3}{c|}{Loa\_CO2} & \multicolumn{3}{c}{Auto-mpg} \\ \cline{2-10}
RBF(Ours)       & 1.00   & 2.59   & 2.59                  & 1.00   & 1.70   & \textbf{1.70}                  & 1.00   & 3.06   & 3.06                   \\
Matérn(Ours)    & 1.00   & 2.56   & \textbf{2.56}         & 1.00   & 2.08   & 2.08                   & 1.00   & 3.24   & 3.24                   \\
Capone22        & 1.00   & 4.62   & 4.62                  & 1.00   & 2.07   & 2.07                   & 1.00   & 2.81   & \textbf{2.81}                   \\
Fiedler21       & 1.00   & 2.57   & 2.57                  & 1.00   & 4.56  & 4.56                  & 1.00   & 7.16   & 7.16                   \\
Lederer19       & 1.00   & 3.07   & 3.07                & 1.00   & 1.73   & 1.73                   & 1.00   & 57.22  & 57.22                   \\ 
\bottomrule
    \end{tabular}
            \caption{Comparison of our method against baselines on real-world datasets of the expected upper and lower bounds for all unseen test points.}
    \label{table:comparison_sup_90}
\end{table*}

\begin{figure}[h]
\begin{center}
\fbox{%
       \begin{minipage}[b]{0.46\textwidth}
            \includegraphics[width=\linewidth]{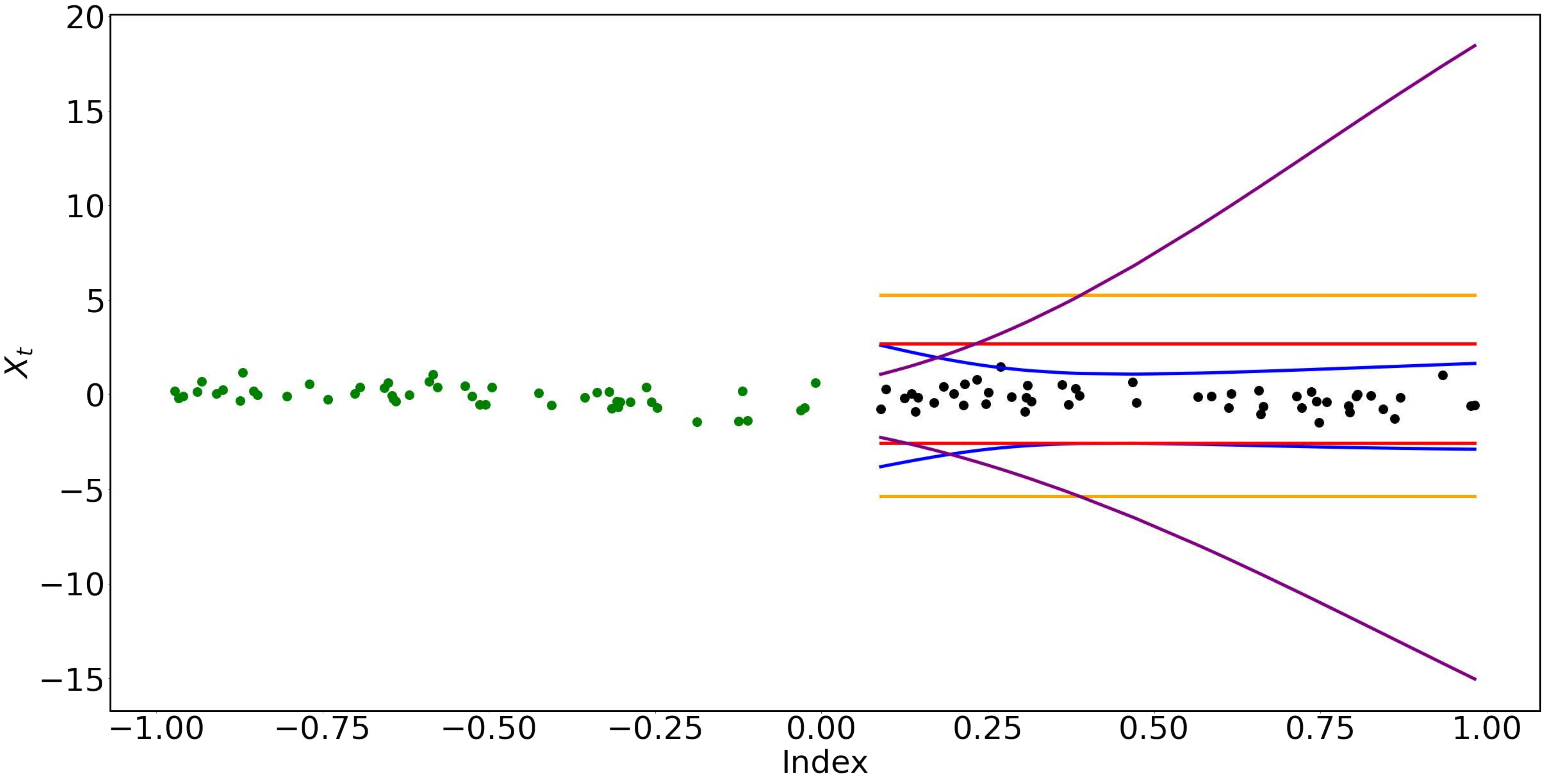}
            \centerline{(a) Synthetic Data}
        \end{minipage}
        \hfill
        \begin{minipage}[b]{0.46\textwidth}
            \includegraphics[width=\linewidth]{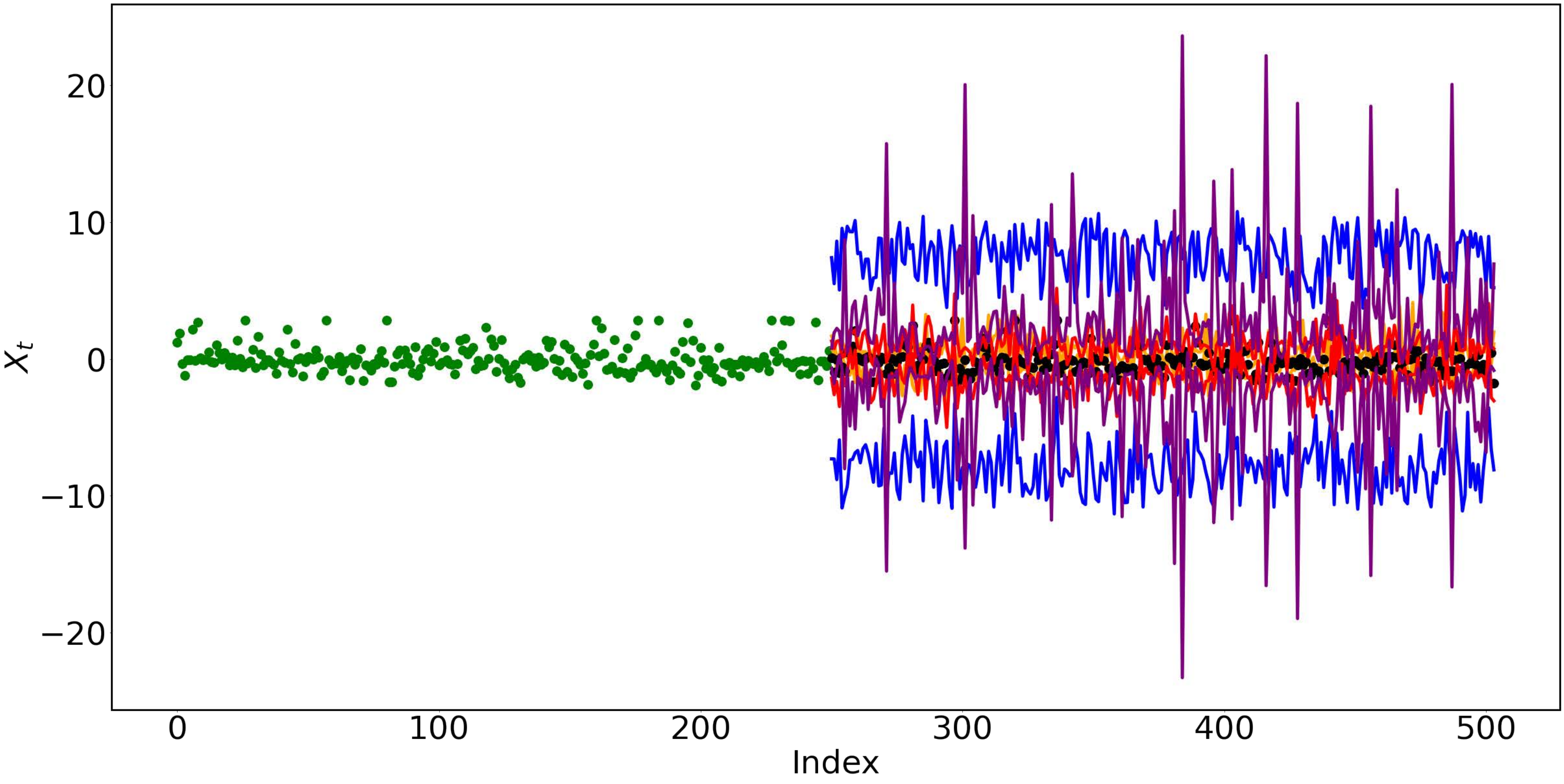}
            \centerline{(b) Boston House Price}
        \end{minipage}
}
\fbox{%
        \begin{minipage}[b]{0.46\textwidth}
            \includegraphics[width=\linewidth]{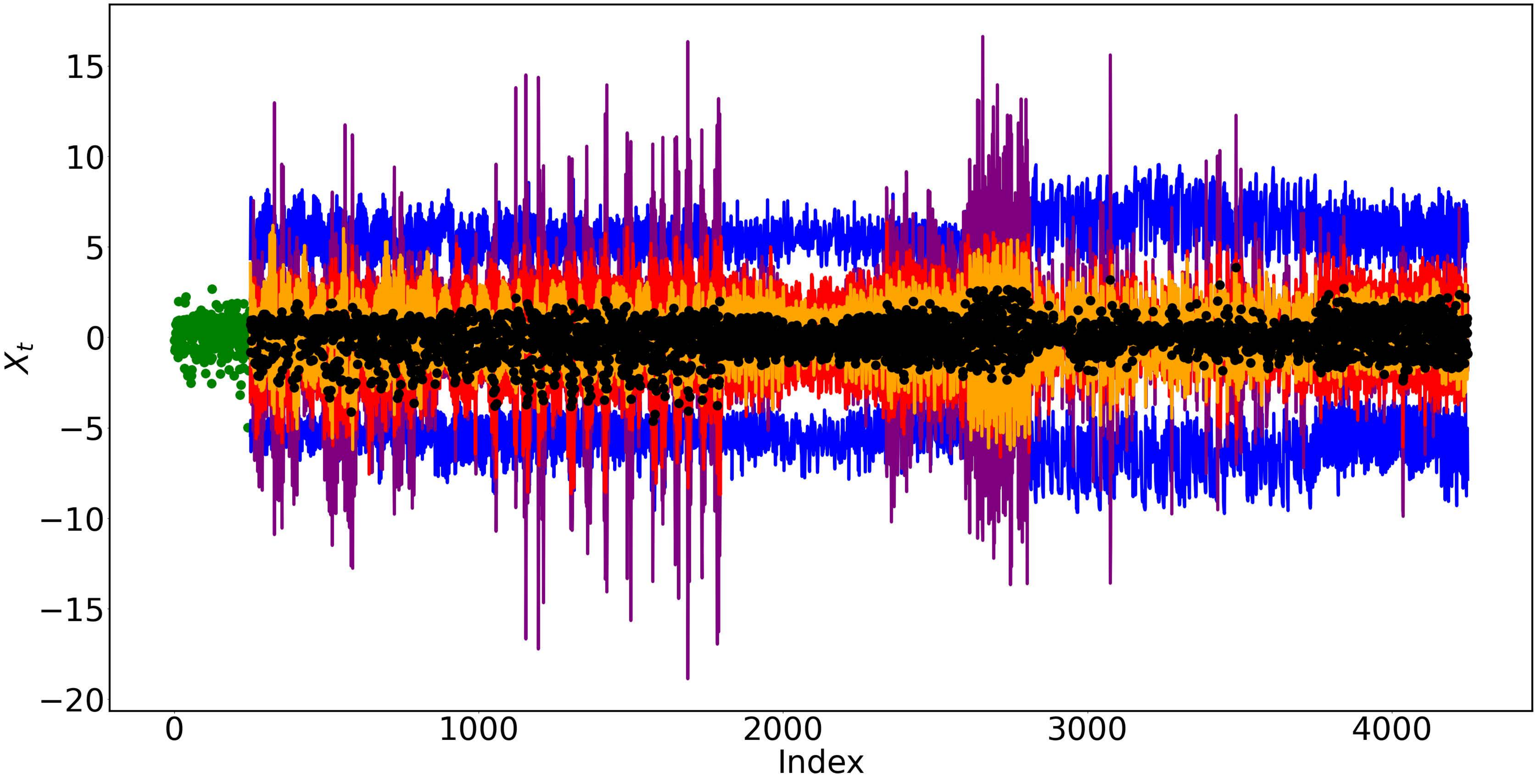}
            \centerline{(c) Sarcos}
        \end{minipage}
        \hfill
        \begin{minipage}[b]{0.46\textwidth}
            \includegraphics[width=\linewidth]{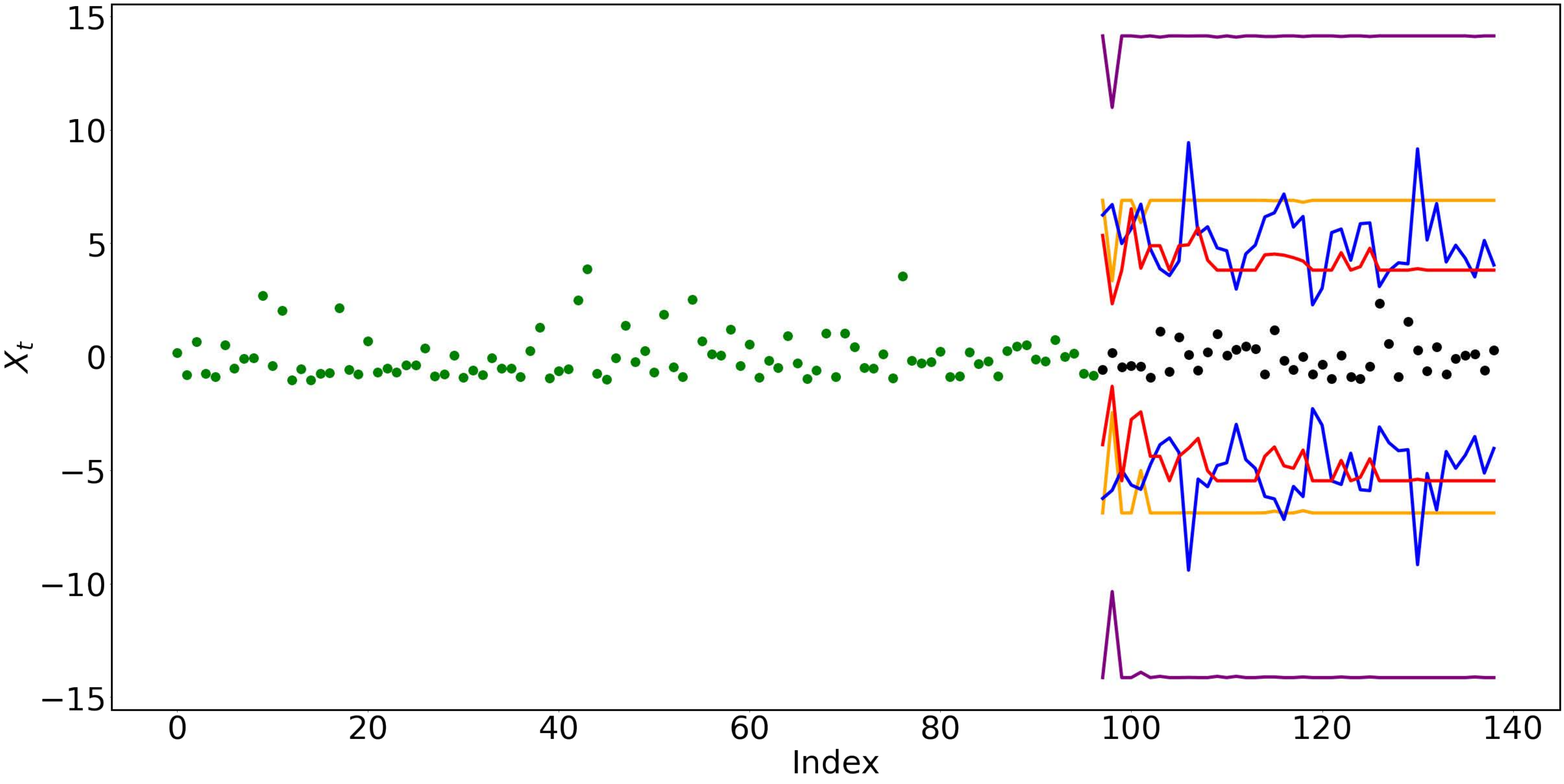}
            \centerline{(d) USGS Earthquake}
        \end{minipage}
}
\fbox{%
        \begin{minipage}[b]{0.46\textwidth}
            \includegraphics[width=\linewidth]{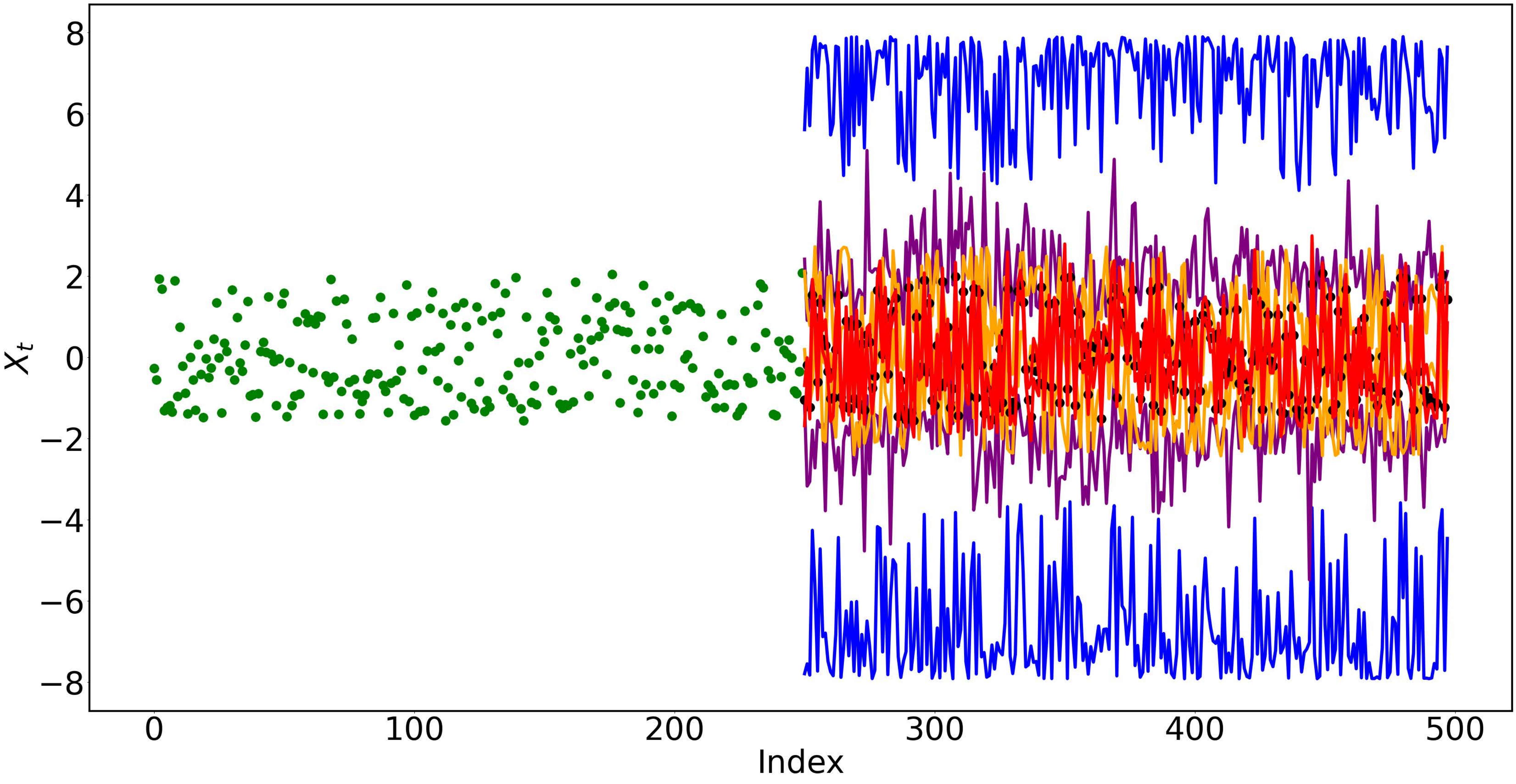}
            \centerline{(e) Loa\_CO2}
        \end{minipage}
        \hfill
        \begin{minipage}[b]{0.46\textwidth}
            \includegraphics[width=\linewidth]{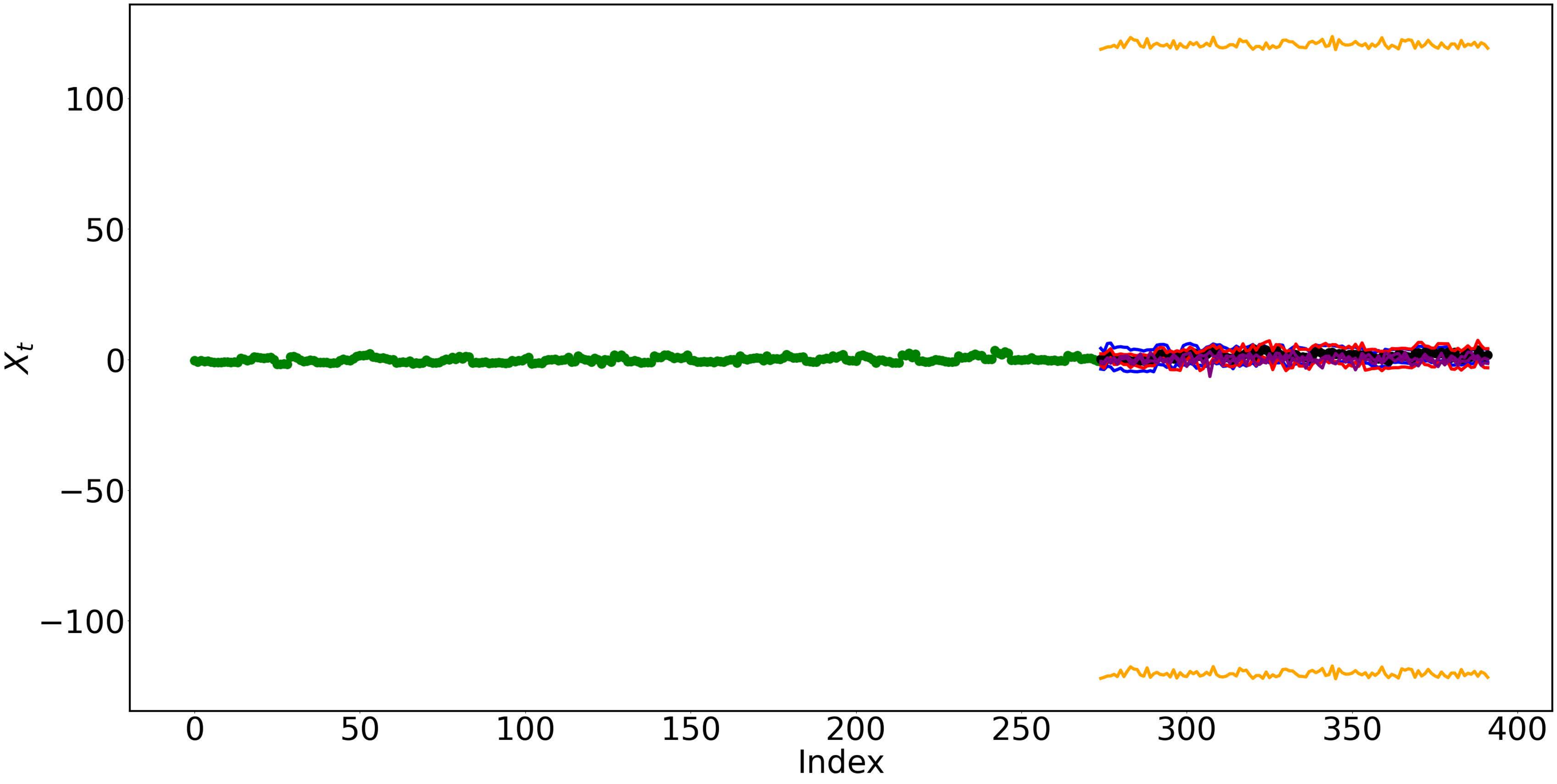}
            \centerline{(f) Auto-mpg}
        \end{minipage}
}
\end{center}
\caption{Comparison of our method with baselines for the test-point-specific bounds. The training set is in green, the test set in black, Lederer19 in orange, Fiedler21 in blue, Capone22 in purple, and our method in red.}
\label{figure:comparison}
\end{figure}

Moreover, for high-dimensional and complex datasets such as the Boston House Price and Sarcos datasets, which exhibit concentrated, high-dimensional characteristics, traditional methods face challenges. In these datasets, the distances between points can vary significantly and in highly nonlinear ways, making it difficult for standard methods to accurately capture the underlying structure. The chaining method, on the other hand, adapts more efficiently to these complexities by focusing on local approximations. This helps prevent the accumulation of errors, resulting in tighter bounds and improved performance, especially for datasets with intricate, high-dimensional relationships.

Figure~\ref{figure:comparison} illustrates this point. In these datasets, the bounds obtained by Capone22 are generally wider than those produced by our chaining method, particularly evident in (d). The bounds generated by Fiedler21 are notably wider than those from our chaining method. Lederer19 produces tighter bounds in most cases, but in (f) with the autompg dataset, the bounds are much wider. This suggests that the chaining method excels in controlling error and delivering more precise bounds, particularly in highly concentrated data environments.

\appsubsection{Statistical Significance}  \label{A24}


To evaluate the statistical significance of our method compared to the baseline models (Fiedler21, Capone22, and Lederer19), we performed permutation t-tests on the CWC (Coverage Width Combination) metric for two different experiments. The first experiment, which includes conventional uncertainty bounds prediction, is shown in Table~\ref{table:t-test-comparison}, while the second experiment, predicting the unseen test set using only the training set, is shown in Table~\ref{table:t-test-comparison-sup}.

\begin{table}[h]
\centering
\setlength{\tabcolsep}{10pt} 
\renewcommand{\arraystretch}{1.1} 
\begin{tabular}{c|ccc}
\hline
BH (Boston Housing) & t-Statistic & p-Value & Statistical Significance \\ \hline
Our Method vs Capone22 & -1.0326  & $<$0.001 & **                       \\ 
Our Method  vs   Fiedler21 & -136.3220    & $<$0.001   & **                       \\ 
Our Method  vs Lederer19 & -12.2456    &  $<$0.001   & **                       \\ 
Our Method  vs Bayesian CI & -3.7533   &  $<$0.001   & **                       \\ 
\hline
Sarcos & t-Statistic & p-Value & Statistical Significance \\ \hline
Our Method vs Capone22 & -2.9995 & $<$0.001 & **                       \\ 
Our Method  vs  Fiedler21 & -86.5386    & $<$0.001   & **                       \\ 
Our Method  vs Lederer19 & -5.1038   &  $<$0.001   & **                       \\ 
Our Method  vs Bayesian CI & -3.4476   &  $<$0.001   & **                       \\ 
\hline
USGS Earthquake & t-Statistic & p-Value & Statistical Significance \\ \hline
Our Method vs Capone22 & -17.4760  & $<$0.001 & **                       \\ 
Our Method  vs   Fiedler21 & -0.9292   & 0.0242  &  *                      \\ 
Our Method  vs Lederer19 & -1.6369 & 0.0027   & **                       \\ 
Our Method  vs Bayesian CI &  -2.2163  &  0.0023   & **                       \\ 
\hline
Loa\_CO2 & t-Statistic & p-Value & Statistical Significance \\ \hline
Our Method vs Capone22 & -2.8560 & $<$0.001 & **                       \\ 
Our Method  vs  Fiedler21 & -164.8383    & $<$0.001   & **                       \\ 
Our Method  vs Lederer19 & -39.9086   &  $<$0.001   & **                       \\ 
Our Method  vs Bayesian CI & -1.1132   &  $<$0.001   & **                       \\ 
\hline
Autompg & t-Statistic & p-Value & Statistical Significance \\ \hline
Our Method vs Capone22 & -3.2506 & $<$0.001 & **                       \\ 
Our Method  vs  Fiedler21 & -11.7752   & $<$0.001   & **                       \\ 
Our Method  vs Lederer19 & -68.3016  &  $<$0.001   & **                       \\ 
Our Method  vs Bayesian CI & -3.6030   &  $<$0.001   & **                       \\ 
\hline
\end{tabular}
\caption{Permutation t-Test comparisons of our method against baselines on real-world data for the test-point-specific bounds.  (** indicates $p < 0.01$; * indicates $p < 0.05$; negative t-statistics indicate that our model performs better than the compared model, as lower CWC values are preferable.)}
\label{table:t-test-comparison}
\end{table}

A permutation t-test is a non-parametric statistical test used to assess whether the performance difference between two models is statistically significant. Unlike traditional t-tests, which assume that the data follows a specific distribution (e.g., normal distribution), permutation tests do not rely on such assumptions, making them suitable for evaluating models when the data distribution is unknown or non-normal. In this context, the permutation t-test enables a more flexible comparison between our method and the baseline models without imposing strict distributional assumptions. By permuting the labels of the data points, the test generates a distribution of differences under the null hypothesis (i.e., that there is no significant difference between the models). This method is particularly beneficial for small datasets or non-normal distributions, providing a robust and reliable measure of statistical significance.

In each trial, the models were trained on the training set and evaluated on the testing set, resulting in 100 independent CWC values for each model. The t-tests were then applied to these CWC values to compare our method with the baselines in Table~\ref{table:t-test-comparison} and Table~\ref{table:t-test-comparison-sup}. This approach ensures that the comparisons account for the variability introduced by the random splits while maintaining the dependency between paired observations. We use \( p < 0.01 \) to indicate high statistical significance, and \( p < 0.05 \) to denote moderate statistical significance. 

Given that lower CWC values correspond to better performance, negative t-statistics suggest that our method generally outperforms the baseline. In the vast majority of cases, our method demonstrates statistically significant improvement over the baseline. However, the USGS Earthquake dataset does not exhibit such significant differences, which can be attributed to the relatively small sample size and limited number of features in this dataset. These factors likely hinder the statistical power of the test, resulting in a lack of significant findings. It is important to note that when sample sizes are small and feature richness is limited, statistical tests often lack the sensitivity needed to detect significant differences, even when performance improvements are evident.

\begin{table}[h]
\centering
\setlength{\tabcolsep}{10pt} 
\renewcommand{\arraystretch}{1.1} 
\begin{tabular}{c|ccc}
\hline
BH (Boston Housing) & t-Statistic & p-Value & Statistical Significance \\ \hline
Our Method vs Capone22 & -6.5777 & $<$0.001 & **                       \\ 
Our Method  vs   Fiedler21 & -106.5292   & $<$0.001   & **                       \\ 
Our Method  vs Lederer19 & -3.1341   &  0.0013   & **                       \\ \hline
Sarcos & t-Statistic & p-Value & Statistical Significance \\ \hline
Our Method vs Capone22 & -16.2167 & $<$0.001 & **                       \\ 
Our Method  vs  Fiedler21 & -39.0211    & $<$0.001   & **                       \\ 
Our Method  vs Lederer19 & -10.7331  &  $<$0.001   & **                       \\ \hline
USGS Earthquake & t-Statistic & p-Value & Statistical Significance \\ \hline
Our Method vs Capone22 & -21.1489  & $<$0.001 & **                       \\ 
Our Method  vs   Fiedler21 & -8.8898   & $<$0.001  &   **                        \\ 
Our Method  vs Lederer19 & -0.8439 & 0.2041   &                      \\ \hline
Loa\_CO2 & t-Statistic & p-Value & Statistical Significance \\ \hline
Our Method vs Capone22 & -13.1667 & $<$0.001 & **                       \\ 
Our Method  vs  Fiedler21 & -80.8352   & $<$0.001   & **                       \\ 
Our Method  vs Lederer19 & -6.5235   &  $<$0.001   & **                       \\ \hline
Autompg & t-Statistic & p-Value & Statistical Significance \\ \hline
Our Method vs Capone22 & 13.1493 & $<$0.001 & **                       \\ 
Our Method  vs  Fiedler21 & -28.0447   & $<$0.001   & **                       \\ 
Our Method  vs Lederer19 & -6.1033  &  $<$0.001   & **                       \\ \hline
\end{tabular}
\caption{Permutation t-test comparisons of our method against baselines on real-world Data for the bound for all unseen test points.  (** indicates $p < 0.01$; * indicates $p < 0.05$; negative t-statistics indicate that our model performs better than the compared model, as lower CWC values are preferable.)}
\label{table:t-test-comparison-sup}
\end{table}

\end{document}